\documentclass{article}

\usepackage[final]{neurips_2021}
\usepackage{amsmath,amsfonts,amsthm}
\usepackage{subfig}
\makeatletter\@input{xx.tex}\makeatother
\makeatother

\usepackage[utf8]{inputenc} 
\usepackage[T1]{fontenc}    
\usepackage{hyperref}       
\usepackage{url}            
\usepackage{booktabs}       
\usepackage{amsfonts}       
\usepackage{nicefrac}       
\usepackage{microtype}      
\usepackage{xcolor}   
\usepackage{graphicx}

\title{Consistent Non-Parametric Methods for Maximizing Robustness}

\author{%
  Robi Bhattacharjee \\
  University of California San Diego\\
  \texttt{rcbhatta@eng.ucsd.edu} \\
   \And
   Kamalika Chaudhuri \\
   University of California San Diego \\
   \texttt{kamalika@eng.ucsd.edu} \\
}
\newtheorem{thm}{Theorem}
\newtheorem{lem}[thm]{Lemma}
\def\supp{supp}
\newtheorem{cor}[thm]{Corollary}
\newtheorem{defn}[thm]{Definition}
\def\D{{\mathcal D}}
\def\U{{\mathcal U}}
\def\V{{\mathcal V}}
\def\X{\mathcal X}
\def\R{\mathbb R}
\def\Y{\{\pm 1\}}
\def\d{\rho}
\def\E{\mathbb{E}}
\def\N{\mathbb{N}}
\def\g{g}

\def\nat{g_{neighbor}}
\def\bad{\D_{1/2}^{-}}
\def\natural{neighborhood preserving}
\def\Natural{Neighborhood preserving}
\def\ncons{neighborhood}
\def\Ncons{Neighborhood}

\begin{document}

\maketitle

\begin{abstract}
Learning classifiers that are robust to adversarial examples has received a
great deal of recent attention. A major drawback of the standard robust
learning framework is there is an artificial robustness radius $r$
that applies to all inputs. This ignores the fact that data may be highly
heterogeneous, in which case it is plausible that robustness regions should be larger in some regions of data, and smaller in others. In this paper, we address this limitation by proposing a new limit classifier, called the neighborhood optimal classifier, that extends the Bayes optimal classifier outside its support by using the label of the closest in-support point. We then argue that this classifier maximizes the size of its robustness regions subject to the constraint of having accuracy equal to the Bayes optimal. We then present sufficient conditions under which general non-parametric methods
that can be represented as weight functions converge towards this limit,
and show that both nearest neighbors and kernel classifiers satisfy them under certain conditions. 
\end{abstract}

\section{Introduction}
Adversarially robust classification, that has been of much recent interest, is typically formulated as follows. We are given data drawn from an underlying distribution $D$, a metric $d$, as well as a pre-specified robustness radius $r$. We say that a classifier $c$ is $r$-robust at an input $x$ if it predicts the same label on a ball of radius $r$ around $x$. Our goal in robust classification is to find a classifier $c$ that maximizes astuteness, which is defined as accuracy on those examples where $c$ is also $r$-robust. 

While this formulation has inspired a great deal of recent work, both theoretical and empirical \cite{Carlini17, Liu17, Papernot17, Papernot16,Szegedy14, Hein17,Schmidt18,Wu16,Steinhardt18, Sinha18, YRSK20}, a major limitation is that enforcing a pre-specified robustness radius $r$ may lead to sub-optimal accuracy {\em{and}} robustness. To see this, consider what would be an ideally robust classifier the example in Figure~\ref{fig:intro}. For simplicity, suppose that we know the data distribution. In this case, a classifier that has an uniformly large robustness radius $r$ will misclassify some points from the blue cluster on the left, leading to lower accuracy. This is illustrated in panel (a), in which large robustness radius leads to intersecting robustness regions. On the other hand, in panel (b), the blue cluster on the right is highly separated from the red cluster, and could be accurately classified with a high margin. But this will not happen if the robustness radius is set small enough to avoid the problems posed in panel (a). Thus, enforcing a fixed robustness radius that applies to the entire dataset may lead to lower accuracy and lower robustness.

In this work, we propose an alternative formulation of robust classification that ensures that in the large sample limit, there is no robustness-accuracy trade off, and that regions of space with higher separation are classified more robustly. An extra advantage is that our formulation is achievable by existing methods. In particular, we show that two very common non-parametric algorithms -- nearest neighbors and kernel classifiers -- achieve these properties in the large sample limit. 

\begin{figure}[ht]
	\centering
	\subfloat[Large robustness radii]{\includegraphics[width=.45\textwidth]{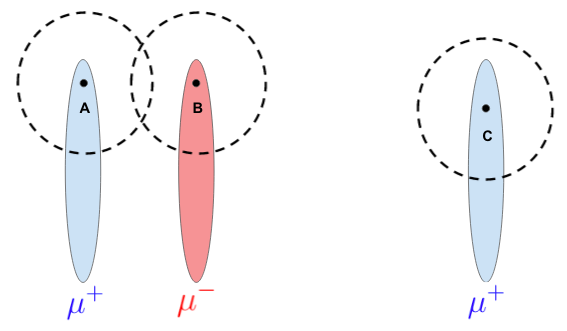}}\hfill
	\subfloat[Small robustness radii]{\includegraphics[width=.45\textwidth]{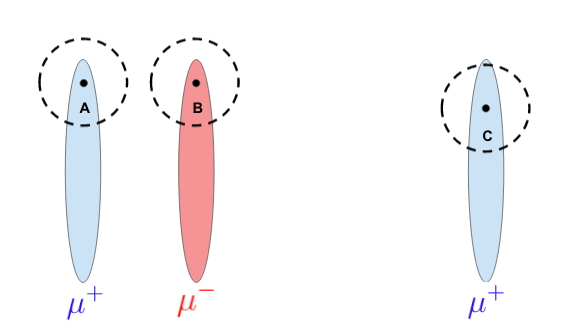}}
	\caption{A data distribution demonstrating the difficulties with fixed radius balls for robustness regions. The red represents negatively labeled points, and the blue positive. If the robustness radius is set too large (panel (a)), then the regions of A and B intersect leading to a loss of accuracy. If the radius is set too small (panel (b)), this leads to a loss of robustness at point C where in principle it should be possible to defend against a larger amount of adversarial attacks.}
	
	\label{fig:intro}
\end{figure}

Our formulation is built on the notion of a new large-sample limit. In the standard statistical learning framework, the large-sample ideal is the Bayes optimal classifier that maximizes accuracy on the data distribution, and is undefined outside. Since this is not always robust with radius $r$, prior work introduces the notion of an $r$-optimal classifier~\cite{YRWC19} that maximizes accuracy on points where it is also $r$-robust. However, this classifier also suffers from the same challenges as the example in Figure~\ref{fig:intro}. 

We depart from both by introducing a new limit that we call the \natural\emph{ }Bayes optimal classifier, described as follows. Given an input $x$ that lies in the support of the data distribution $D$, it predicts the same label as the Bayes optimal. On an $x$ outside the support, it outputs the prediction of the Bayes Optimal on the nearest neighbor of $x$ {\em{within}} the support of $D$. The first property ensures that there is no loss of accuracy -- since it always agrees with the Bayes Optimal within the data distribution. The second ensures higher robustness in regions that are better separated. Our goal is now to design classifiers that converge to the \natural\emph{ }Bayes optimal in the large sample limit; this ensures that with enough data, the classifier will have accuracy approaching that of the Bayes optimal, as well as higher robustness where possible without sacrificing accuracy. 

We next investigate how to design classifiers with this convergence property. Our starting point is classical statistical theory~\cite{Stone77} that shows that a class of methods known as weight functions will converge to a Bayes optimal in the large sample limit provided certain conditions hold; these include $k$-nearest neighbors under certain conditions on $k$ and $n$, certain kinds of decision trees as well as kernel classifiers. Through an analysis of weight functions, we next establish precise conditions under which they converge to the \natural\emph{ }Bayes optimal in the large sample limit. As expected, these are stronger than standard convergence to the Bayes optimal. In the large sample limit, we show that $k_n$-nearest neighbors converge to the \natural\emph{ }Bayes optimal provided $k_n = \omega(\log n)$, and kernel classifiers converge to the \natural\emph{ }Bayes optimal provided certain technical conditions (such as the bandwidth shrinking sufficiently slowly). By contrast, certain types of histograms do not converge to the \natural\emph{ }Bayes optimal, even if they do converge to the Bayes optimal.  We round these off with a lower bound that shows that for nearest neighbor, the condition that $k_n = \omega(\log n)$ is tight. In particular, for $k_n = O(\log n)$, there exist distributions for which $k_n$-nearest neighbors provably fails to converge towards the \natural\emph{ }Bayes optimal (despite converging towards the standard Bayes optimal).


In summary, the contributions of the paper are as follows. First, we propose a new large sample limit the \natural\emph{ }Bayes optimal and a new formulation for robust classification. We then establish conditions under which weight functions, a class of non-parametric methods, converge to the \natural\emph{ }Bayes optimal in the large sample limit. Using these conditions, we show that $k_n$-nearest neighbors satisfy these conditions when $k_n = \omega(\log n)$, and kernel classifiers satisfy these conditions provided the kernel function $K$ has faster than polynomial decay, and the bandwidth parameter $h_n$ decreases sufficiently slowly. 

To complement these results, we also include negative examples of non-parametric classifiers that do not converge. We provide an example where histograms do not converge to the \natural\emph{ }Bayes optimal with increasing $n$. We also show a lower bound for nearest neighbors, indicating that $k_n = \omega(\log n)$ is both necessary and sufficient for convergence towards the \natural\emph{ }Bayes optimal. 

Our results indicate that the \natural\emph{ }Bayes optimal formulation shows promise and has some interesting theoretical properties. We leave open the question of coming up with other alternative formulations that can better balance both robustness and accuracy for all kinds of data distributions, as well as are achievable algorithmically. We believe that addressing this would greatly help address the challenges in adversarial robustness.

\section{Preliminaries}

We consider binary classification over $\R^d \times \Y$, and let $\d$ denote any distance metric on $\R^d$. We let $\mu$ denote the measure over $\R^d$ corresponding to the probability distribution over which instances $x \in \R^d$ are drawn. Each instance $x$ is then labeled as $+1$ with probability $\eta(x)$  and $-1$ with probability $1 - \eta(x)$. Together, $\mu$ and $\eta$ comprise our data distribution $\D = (\mu, \eta)$ over $\R^d \times \Y$.

For comparison to the robust case, for a classifier $f: \R^d \to \{\pm 1\}$ and a distribution $\D$ over $\R^d \times \{\pm 1\}$, it will be instructive to consider its  \textbf{accuracy},  denoted $A(f, \D)$, which is defined as the fraction of examples from $\D$ that $f$ labels correctly. Accuracy is maximized by the \textbf{Bayes Optimal classifier}: which we denote by $\g$. It can be shown that for any $x \in \supp(\mu)$, $\g(x) = 1$ if $\eta(x) \geq \frac{1}{2}$, and $\g(x) = -1$ otherwise. 

Our goal is to build classifiers $\R^d \to \Y$ that are both accurate and robust to small perturbations. For any example $x$, perturbations to it are constrained to taking place in the \textbf{robustness region} of $x$, denoted $U_x$. We will let $\U = \{U_x: x \in \R^d\}$ denote the collections of all robustness regions. 

We say that a classifier $f: \R^d \to \{\pm 1\}$ is \textbf{robust} at $x$ if for all $x' \in U_x$, $f(x') = f(x)$. Combining robustness and accuracy, we say that classifier is \textbf{astute} at a point $x$ if it is both accurate and robust. Formally, we have the following definition. 

\begin{defn}
A classifier $f: \R^d \to \Y$ is said to be \textbf{astute} at $(x,y)$ with respect to robustness collection $\U$ if $f(x) = y$ and $f$ is robust at $x$ with respect to $\U$. If $\D$ is a data distribution over $\R^d \times \Y$, the \textbf{astuteness} of $f$ over $\D$ with respect to $\U$, denoted $A_\U(f, \D)$, is the fraction of examples $(x,y) \sim \D$ for which $f$ is astute at $(x,y)$ with respect to $\U$. Thus $$A_\U(f, \D) = P_{(x, y) \sim \D}[f(x') = y, \forall x' \in \U_x].$$
\end{defn}

\paragraph{Non-parametric Classifiers} We now briefly review several kinds of non-parametric classifiers that we will consider throughout this paper. We begin with \textit{weight functions}, which are a general class of non-parametric algorithms that encompass many classic algorithms, including nearest neighbors and kernel classifiers.

\textbf{Weight functions} are built from training sets, $S = \{(x_1, y_1), (x_2, y_2,), \dots, (x_n, y_n)\}$ by assigning a function $w_i^S: \R^d \to [0, 1]$ that essentially scores how relevant the training point $(x_i, y_i)$ is to the example being classified. The functions $w_i^S$ are allowed to depend on $x_1, \dots, x_n$ but must be independent of the labels $y_1, \dots, y_n$. Given these functions, a point $x$ is classified by just checking whether $\sum y_iw_i^S(x) \geq 0$ or not. If it is nonnegative, we output $+1$ and otherwise $-1$. A complete description of weight functions is included in the appendix. 

Next, we enumerate several common Non-parametric classifiers that can be construed as weight functions. Details can be found in the appendix.

\textbf{Histogram classifiers} partition the domain $\R^d$ into cells recursively by splitting cells that contain a sufficiently large number of points $x_i$. This corresponds to a weight function in which $w_i^S(x) = \frac{1}{k_x}$ if $x_i$ is in the same cell as $x$, where $k_x$ denotes the number of points in the cell containing $x$.

$k_n$-\textbf{nearest neighbors} corresponds to a weight function in which $w_i^S(x) = \frac{1}{k_n}$ if $x_i$ is one of the $k_n$ nearest neighbors of $x$, and $w_i^S(x) = 0$ otherwise.

\textbf{Kernel-Similarity classifiers} are weight functions built from a kernel function $K:\R_{\geq 0} \to \R_{\geq 0}$ and a window size $(h_n)_1^\infty$ such that $w_i^S(x) \propto  K(\d(x, x_i)/h_n)$ (we normalize by dividing by $\sum_1^n K((\d(x, x_i)/h_n))$).
\section{The \Natural\emph{ }Bayes optimal classifier}

Robust classification is typically studied by setting the robustness regions,  $\mathcal{U} = \{U_x\}_{x \in \R^d}$, to be balls of radius $r$ centered at $x$, $U_x = \{x': \d(x, x') \leq r\}$. The quantity $r$ is the robustness radius, and is typically set by the practitioner (before any training has occurred). 

This method has a limitation with regards to trade-offs between accuracy and robustness. To increase the margin or robustness, we must have a large robustness radius (thus allowing us to defend from larger adversarial attacks). However, with large robustness radii, this can come at a cost of accuracy, as it is not possible to robustly give different labels to points with intersecting robustness regions. 

For an illustration, consider Figure \ref{fig:intro}. Here we consider a data distribution $D = (\mu, \eta)$ in which the blue regions denote all points with $\eta(x) > 0.5$ (and thus should be labeled $+$), and the red regions denote all points with $\eta(x) < 0.5$ (and thus should be labeled $-$). Observe that it is not possible to be simultaneously accurate and robust at points $A, B$ while enforcing a large robustness radius, as demonstrated by the intersecting balls. While this can be resolved by using a smaller radius, this results in losing out on potential robustness at point $C$. In principal, we should be able to afford a large margin of robustness about $C$ due to its relatively far distance from the red regions. 

Motivated by this issue, we seek to find a formalism for robustness that allows us to simultaneously avoid paying for any accuracy-robustness trade-offs and \textit{adaptively} size robustness regions (thus allowing us to defend against a larger range of adversarial attacks at points that are located in more homogenous zones of the distribution support). To approach this, we will first provide an ideal limit object: a classifier that has the same accuracy as the Bayes optimal (thus meeting our first criteria) that has good robustness properties. We call this the the \natural\emph{ }Bayes optimal classifier, defined as follows.

\begin{defn}
Let $\D = (\mu, \eta)$ be a distribution over $\R^d \times \{\pm 1\}$. Then the \textbf{\natural\emph{ }Bayes optimal classifier of $\D$}, denoted $\nat$, is the classifier defined as follows. Let $\mu^+ = \{x: \eta(x) \geq \frac{1}{2}\}$ and $\mu^- = \{x: \eta(x) < \frac{1}{2}\}$. Then for any $x \in \R^d$, $\nat(x) = +1$ if $\d(x, \mu^+) \leq \d(x, \mu^-)$, and $\nat(x) = -1$ otherwise. 
\end{defn}

This classifier can be thought of as the most robust classifier that matches the accuracy of the Bayes optimal. We call it \textit{\natural} because it extends the Bayes optimal classifier into a local neighborhood about every point in the support. For an illustration, refer to Figure \ref{fig:decision_boundary2}, which plots the decision boundary of the \natural\emph{ }Bayes optimal for an example distribution. 

\begin{figure}
    \centering
        \includegraphics[scale=0.30] {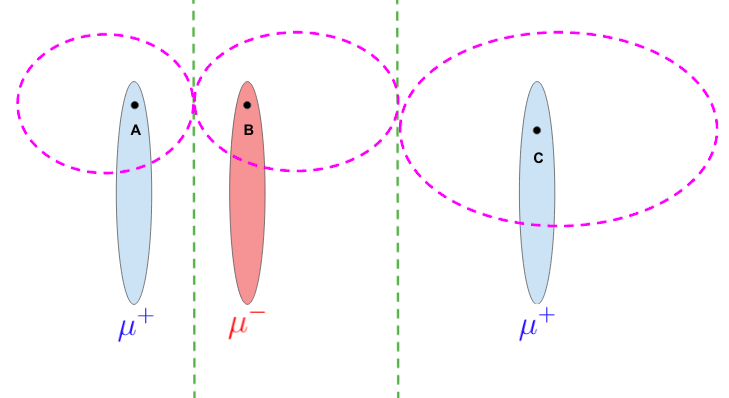}
    \caption{The decision boundary of the \natural\emph{ }Bayes optimal classifier is shown in green,  and the \natural\emph{ }robust region of $x$ is shown in pink. The former consists of points equidistant from $\mu^+, \mu^-$, and the latter consists of points equidistant from $x$, $\mu^+$.}
    \label{fig:decision_boundary2}
\end{figure}

Next, we turn our attention towards measuring its robustness, which must be done with respect to some set of robustness regions $\mathcal{U} = \{U_x\}$. While these regions $U_x$ can be nearly arbitrary, we seek regions $U_x$ such that $A_\U(g_{max}, \D) = A(g_{bayes}, \D)$ (our astuteness equals the maximum possible accuracy) and $U_x$ are ``as large as possible" (representing large robustness). To this end, we propose the following regions.

\begin{defn}\label{def:nat_region}
Let $\D = (\mu, \eta)$ be a data distribution over $\R^d \times \{\pm 1\}$. Let $\mu^+ = \{x: \eta(x) > \frac{1}{2}\}$, $\mu^- = \{x: \eta(x) < \frac{1}{2}\}$, and $\mu^{1/2} = \{x: \eta(x) = \frac{1}{2}\}$. For $x \in \mu^+$, we define the \textbf{\natural\emph{ }robustness region}, denoted $V_x$, as $$V_x = \{x': \rho(x, x') < \rho(\mu^- \cup \mu^{\frac{1}{2}}, x')\}.$$ It consists of all points that are closer to $x$ than they are to $\mu^- \cup \mu^{1/2}$ (points oppositely labeled from $x$). We can use a similar definition for $x \in \mu^{-}$. Finally, if $x \in \mu^{1/2}$, we simply set $V_x = \{x\}$. 
\end{defn}

These robustness regions take advantage of the structure of the \natural\emph{ }Bayes optimal. They can essentially be thought of as regions that maximally extend from any point $x$ in the support of $\D$ to the decision boundary of the \natural\emph{ }Bayes optimal. We include an illustration of the regions $V_x$ for an example distribution in Figure \ref{fig:decision_boundary2}. 

As a technical note, for $x \in supp(\D)$ with $\eta(x) = 0.5$, we give them a trivial robustness region. The rational for doing this is that $\eta(x) = 0.5$ is an edge case that is arbitrary to classify, and consequently enforcing a robustness region at that point is arbitrary and difficult to enforce. 

We now formalize the robustness and accuracy guarantees of the max-margin Bayes optimal classifier with the following two results.

\begin{thm}\label{thm:accuracy_margin}
(Accuracy) Let $\D$ be a data distribution. Let $\V$ denote the collection of \natural\emph{ }robustness regions, and let $g$ denote the Bayes optimal classifier. Then the \natural\emph{ }Bayes optimal classifier, $\nat$, satisfies $A_\V(\nat, \D) = A(g, \D)$, where $A(g, \D)$ denotes the accuracy of the Bayes optimal. Thus, $\nat$ maximizes accuracy.
\end{thm}

\begin{thm}\label{thm:robust_margin}
(Robustness) Let $\D$ be a data distribution, let $f$ be a classifier, and let $\U$ be a set of robustness regions. Suppose that $A_\U(f, \D) = A(g, \D)$, where $g$ denotes the Bayes optimal classifier. Then there exists $x \in \supp(\D)$ such that $V_x \not \subset U_x$, where $V_x$ denotes the \natural\emph{ }robustness region about $x$. In particular, we cannot have $V_x$ be a strict subset of $U_x$ for all $x$. 
\end{thm}

Theorem \ref{thm:accuracy_margin} shows that the \natural\emph{ }Bayes classifier achieves maximal accuracy, while Theorem \ref{thm:robust_margin} shows that achieving a strictly higher robustness (while maintaining accuracy) is not possible; while it is possible to make accurate classifiers which have higher robustness than $\nat$ in some regions of space, it is not possible for this to hold across all regions. Thus, the \natural\emph{ }Bayes optimal classifier can be thought of as a local maximum to the constrained optimization problem of maximizing robustness subject to having maximum (equal to the Bayes optimal) accuracy. 

\subsection{\Ncons\emph{ }Consistency}

Having defined the \natural\emph{ }Bayes optimal classifier, we now turn our attention towards building classifiers that converge towards it. Before doing this, we must precisely define what it means to converge. Intuitively, this consists of building classifiers whose robustness regions ``approach" the robustness regions of the \natural\emph{ }Bayes optimal classifier. This motivates the definition of \textit{partial \natural\emph{ }robustness regions}.
\begin{defn}\label{def:partial_nat_region}
Let $0 < \kappa < 1$ be a real number, and let $\D = (\mu, \eta)$ be a data distribution over $\R^d \times \{\pm 1\}$. Let $\mu^+ = \{x: \eta(x) > \frac{1}{2}\}$, $\mu^- = \{x: \eta(x) < \frac{1}{2}\}$, and $\mu^{1/2} = \{x: \eta(x) = \frac{1}{2}\}$. For $x \in \mu^+$, we define the \textbf{\natural\emph{ }robustness region}, denoted $V_x$, as $$V_x = \{x': \rho(x, x') < \kappa\rho(\mu^- \cup \mu^{\frac{1}{2}}, x')\}.$$ It consists of all points that are closer to $x$ than they are to $\mu^- \cup \mu^{1/2}$ (points oppositely labeled from $x$) by a factor of $\kappa$. We can use a similar definition for $x \in \mu^{-}$. Finally, if $\eta(x) = \frac{1}{2}$, we simply set $V_x^\kappa = \{x\}$.
\end{defn}

Observe that $V_x^{\kappa} \subset V_x$ for all $0 < \kappa < 1$, and thus being robust with respect to $V_x^{\kappa}$ is a milder condition than $V_x$. Using this notion, we can now define margin consistency.

\begin{defn}
A learning algorithm $A$ is said to be \textbf{\ncons\emph{ }consistent} if the following holds for any data distribution $\D$. For any $0 < \epsilon, \delta, \kappa < 1$, there exists $N$ such that for all $n \geq N$, with probability at least $1- \delta$ over $S \sim \D^n$, $$A_{\V^\kappa}(A_S, D) \geq A(g, \D) - \epsilon,$$  where $g$ denotes the Bayes optimal classifier and $A_S$ denotes the classifier learned by algorithm $A$ from dataset $S$. 
\end{defn}

This condition essentially says that the astuteness of the classifier learned by the algorithm converges towards the accuracy of the Bayes optimal classifier. Furthermore, we stipulate that this holds as long as the astuteness is measured with respect to some $\V^\kappa$. Observe that as $\kappa \to 1$, these regions converge towards the \natural\emph{ }robustness regions, thus giving us a classifier with robustness effectively equal to that of the \natural\emph{ }Bayes optimal classifier. 

\section{\Ncons\emph{ }Consistent Non-Parametric Classifiers}

Having defined \ncons\emph{ }consistency, we turn to the following question: which non-parametric algorithms are \ncons\emph{ }consistent? Our starting point will be the standard literature for the convergence of non-parametric classifiers with regard to accuracy. We begin by considering the standard conditions for $k_n$-nearest neighbors to converge (in accuracy) towards the Bayes optimal.

$k_n$-nearest neighbors is \textit{consistent} if and only if the following two conditions are met:  $\lim_{n \to \infty} k_n = \infty$, and $\lim_{n \to \infty} \frac{k_n}{n} = 0$. The first condition guarantees that each point is classified by using an increasing number of nearest neighbors (thus making the probability of a misclassification small), and the second condition guarantees that each point is classified using only points very close to it. We will refer to the first condition as \textit{precision}, and the second condition as \textit{locality.}  A natural question is whether the same principles suffice for \ncons\emph{ }consistency as well. We began by showing that without any additional constraints, the answer is no.

\begin{thm}\label{thm:lower_bound}
Let $\D = (\mu, \eta)$ be the data distribution where $\mu$ denotes the uniform distribution over $[0,1]$ and $\eta$ is defined as: $\eta(x) = x$. Over this space, let $\d$ be the euclidean distance metric. Suppose $k_n = O(\log n)$ for $1 \leq n < \infty$. Then $k_n$-nearest neighbors is not \ncons\emph{ }consistent with respect to $\D$. 
\end{thm}

The issue in the example above is that for smaller $k_n$, $k_n$-nearest neighbors lacks sufficient precision. For \ncons\emph{ }consistnecy, points must be labeled using even more training points than are needed accuracy. This is because the classifier must be uniformly correct across the entirety of $V_x^\kappa$. Thus, to build \ncons\emph{ }consistent classifiers, we must bolster the precision from the standard amount used for standard consistency. To do this, we begin by introducing \textit{splitting numbers}, a useful tool for bolstering the precision of weight functions.

\subsection{Splitting Numbers}

We will now generalize beyond nearest neighbors to consider weight functions. Doing so will allow us to simultaneously analyze nearest neighbors and kernel classifiers. To do so, we must first rigorously substantiate our intuitions about increasing precision into concrete requirements. This will require several technical definitions.

\begin{defn}\label{defn:prob_radius}
Let $\mu$ be a probability measure over $\R^d$. For any $x \in \R^d$, the \textbf{probability radius} $r_p(x)$ is the smallest radius for which $B(x, r_p(x))$ has probability mass at least $p$. More precisely, $r_p(x) = \inf\{r: \mu(B(x,r)) \geq p\}.$ 
\end{defn}

\begin{defn}\label{defn:splitting_number}
Let $W$ be a weight function and let $S = \{x_1, x_2, \dots, x_n\}$ be any finite subset of $\R^d$. For any $x \in \R^d$, $\alpha \geq 0$, and $0 \leq \beta \leq 1$, let $W_{x, \alpha, \beta}  = \{i: \d(x, x_i) \leq \alpha, w_i^S(x) \geq \beta\}.$ Then the \textbf{splitting number} of $W$ with respect to $S$, denoted as $T(W, S)$ is the number of distinct subsets generated by $W_{x, \alpha \beta}$ as $x$ ranges over $\R^d$, $\alpha$ ranges over $[0, \infty)$, and $\beta$ ranges over $[0,1]$. Thus $T(W, S) = |\{W_{x, \alpha, \beta}: x \in \R^d, 0 \leq \alpha, 0 \leq \beta \leq 1\}|.$
\end{defn}

Splitting numbers allow us to ensure high amounts of precision over a weight function. To prove \ncons\emph{ }consistency, it is necessary for a classifier to be correct at \textit{all} points in a given region. Consequently, techniques that consider a single point will be insufficient. The splitting number provides a mechanism for studying entire regions simultaneously. For more details on splitting numbers, we include several examples in the appendix. 

\subsection{Sufficient Conditions for Neighborhood Consistency}

We now state our main result.
\begin{thm}\label{thm:main}
Let $W$ be a weight function, $\D$ a distribution over $\R^d \times \{\pm 1\}$, $\U$ a neighborhood preserving collection, and $(t_n)_1^{\infty}$ be a sequence of positive integers such that the following four conditions hold. 

1. $W$ is consistent (with resp. to accuracy) with resp. to $\D$.

2. For any $0 < p < 1$, $\lim_{n \to \infty} E_{S \sim \D^n} [\sup_{x \in \R^d} \sum_1^n w_i^S(x)1_{\d(x, x_i) > r_p(x)}] = 0.$

3. $\lim_{n \to \infty} E_{S \sim D^n}[t_n \sup_{x \in \R^d} w_i^S(x)] = 0$.

4. $\lim_{n \to \infty} E_{S \sim D^n}\frac{\log T(W,S)}{t_n} = 0$.

Then $W$ is \ncons\emph{ }consistent with respect to $\D$.
\end{thm}

\textbf{Remarks:} Condition 1 is necessary because \ncons\emph{ }consistency implies standard consistency -- or, convergence in accuracy to the Bayes Optimal. Standard consistency has been well studied for non-parametric classifiers, and there are a variety of results that can be used to ensure it -- for example, Stone's Theorem (included in the appendix). 

Conditions 2. and 3. are stronger version of conditions 2. and 3. of Stone's theorem. In particular, both include a supremum taken over all $x \in \R^d$ as opposed to simply considering a random point $x \sim \D$. This is necessary for ensuring correct labels on entire regions of points simultaneously. We also note that the dependence on $r_p(x)$ (as opposed to some fixed $r$) is a key property used for adaptive robustness. This allows the algorithm to adjust to potential differing distance scales over different regions in $\R^d$. This idea is reminiscent of the analysis given in \cite{Dasgupta14}, which also considers probability radii.

Condition 4. is an entirely new condition which allows us to simultaneously consider all $T(W,S)$ subsets of $S$. This is needed for analyzing weighted sums with arbitrary weights.

Next, we apply Theorem \ref{thm:main} to get specific examples of margin consistent non-parametric algorithms.

\subsection{Nearest Neighbors and Kernel Classifiers}

We now provide sufficient conditions for $k_n$-nearest neighbors to be  \ncons\emph{ }consistent.

\begin{cor}\label{cor:nn}
Suppose $(k_n)_1^{\infty}$ satisfies (1) $\lim_{n \to \infty} \frac{k_n}{n} = 0$, and (2) $\lim_{n \to \infty} \frac{\log n}{k_n} = 0$. Then $k_n$-nearest neighbors is \ncons\emph{ }consistent.
\end{cor}

As a result of Theorem \ref{thm:lower_bound}, corollary \ref{cor:nn} is tight for nearest neighbors. Thus $k_n$ nearest neighbors is \ncons\emph{ }consistent if and only if $k_n = \omega(\log n)$. 

Next, we give sufficient conditions for a kernel-similarity classifier.
\begin{cor}\label{cor:kern}
Let $W$ be a kernel classifier over $\R^d \times \Y$ constructed from $K: \R^+ \to \R^+$ and $h_n$. Suppose the following properties hold.

1. $K$ is decreasing, and satisfies $\int_{\R^d}K(||x||)dx < \infty.$

2. $\lim_{n \to \infty} h_n = 0$ and $\lim_{n \to \infty} nh_n^d = \infty$.

3. For any $c > 1$, $\lim_{x \to \infty} \frac{K(cx)}{K(x)} = 0$.

4. For any $x \geq 0$, $\lim_{n \to \infty} \frac{n}{\log n}K(\frac{x}{h_n}) = \infty$.

Then $W$ is \ncons\emph{ }consistent.
\end{cor}

Observe that conditions 1. 2. and 3. are satisfied by many common Kernel functions such as the Gaussian or Exponential kernel ($K(x) = \exp(-x^2)$/ $K(x) = \exp(-x)$). Condition 4. can be similarly satisfied by just increasing $h_n$ to be sufficiently large. Overall, this theorem states that Kernel classification is \ncons\emph{ }consistent as long as the bandwidth shrinks slowly enough.

\begin{figure}
    \centering
        \includegraphics[scale=0.25] {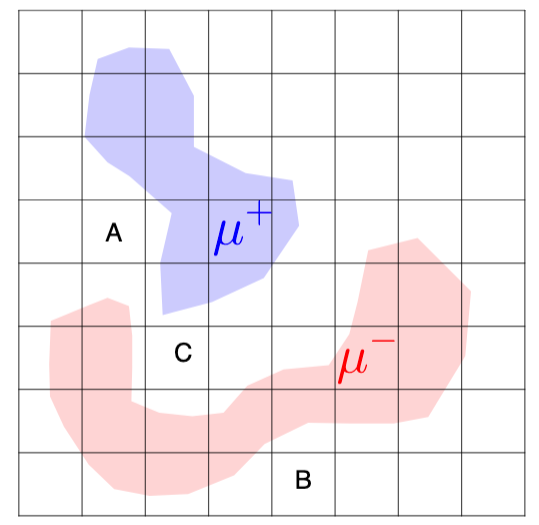}
    \caption{we have a histogram classifier being applied to the blue and red regions. The classifier will be unable to construct good labels in the cells labeled $A, B, C$, and consequently will not be robust with respect to $V_x^{\kappa}$ for sufficiently large $\kappa$.}
    \label{fig:histogram}
\end{figure}

\subsection{Histogram Classifiers}

Having discussed \ncons\emph{ }consistent nearest-neighbors and kernel classifier, we now turn our attention towards another popular weight function, histogram classifiers. Recall that histogram classifiers operate by partitioning their input space into increasingly small cells, and then classifying each cell by using a majority vote from the training examples within that cell (a detailed description can be found in the appendix). We seek to answer the following question: is increasing precision sufficient for making histogram classifiers \ncons\emph{ }consistent? Unfortunately, the answer this turns out not to be no. The main issue is that histogram classifiers have no mechanism for performing classification outside the support of the data distribution. 

For an example of this, refer to Figure \ref{fig:histogram}. Here we see a distribution being classified by a histogram classifier. Observe that the cell labeled $A$ contains points that are strictly closer to $\mu^+$ than $\mu^-$, and consequently, for sufficiently large $\kappa$, $V_x^{\kappa}$ will intersect $A$ for some point $x \in \mu^+$. A similar argument holds for the cells labeled $B$ and $C.$. However, since $A, B, C$ are all in cells that will never contain any data, they will never be labeled in a meaningful way. Because of this, histogram classifiers are not \ncons\emph{ }consistent.


\section{Validation}

\begin{figure}[ht]
	\centering
	\subfloat[exponential kernel]{\includegraphics[width=.4\textwidth]{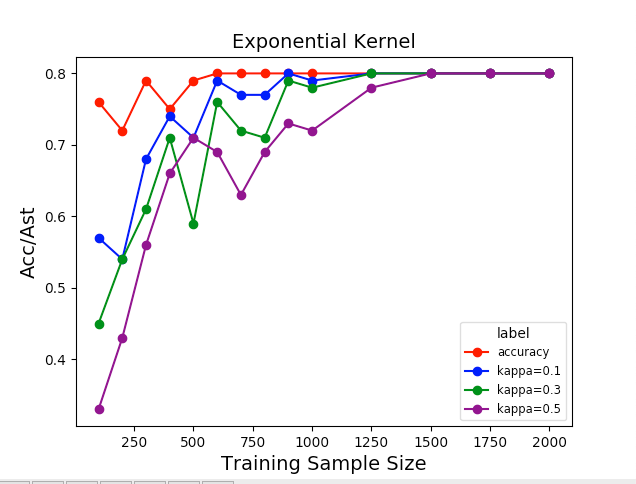}}
	\subfloat[polynomial kernel]{\includegraphics[width=.4\textwidth]{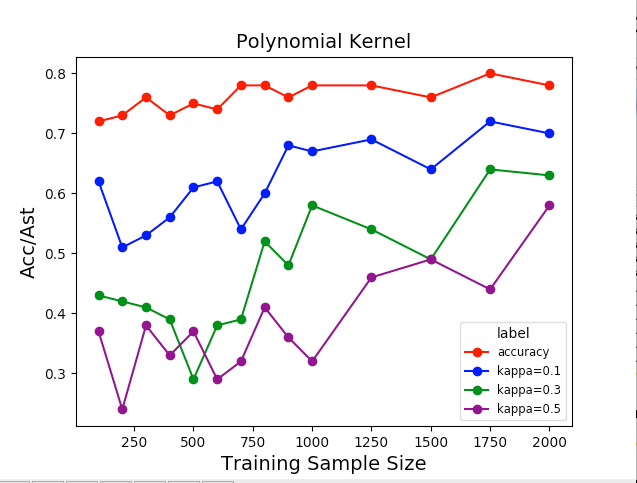}}
	\caption{Plots of astuteness against the training sample size. In both panels, accuracy is plotted in red, and the varying levels of robustness regions $(\kappa = 0.1, 0.3, 0.5)$ are givne in blue, green and purple. In panel (a), observe that as sample size increases, every measure of astuteness converges towards $0.8$ which is as predicted by Corollary \ref{cor:kern}. In panel (b), although the accuracy appears to converge, none of the robustness measure. In fact, they get progressively worse the larger $\kappa$ gets. 
	}
	\label{fig:validation}
\end{figure}

To complement our theoretical large sample results for non-parametric classifiers, we now include several experiments to understand their behavior for finite samples. We seek to understand how quickly non-parametic classifiers converge towards the \natural\emph{ }Bayes optimal. 

We focus our attention on kernel classifiers and use two different kernel similarity functions: the first, an exponential kernel, and the second, a polynomial kernel.  These classifiers were chosen so that the former meets the conditions of Corollary \ref{cor:kern}, and the latter does not. Full details on these classifiers can be found in the appendix.

To be able to measure performance with increasing data size, we look at a simple synthetic dataset over overlayed circles (see Figure \ref{fig:distribution} for an illustration) with support designed so that the data is intrinsically multiscaled. In particular, this calls for different levels of robustness in different regions. For simplicity, we use a global label noise parameter of $0.2$, meaning that any sample drawn from this distribution is labeled differently than its support with probability $0.2$. Further details about our dataset are given in section \ref{sec:experiment_details}. 

\textbf{Performance Measure.} For a given classifier, we evaluate its astuteness at a test point $x$ with respect to the robustness region $V_x^{\kappa}$ (Definition \ref{def:partial_nat_region}). While these regions are not computable in practice due to their dependency on the  support of the data distribution, we are able to approximate them for this synthetic example due to our explicit knowledge of the data distribution. Details for doing this can be found in the appendix. To compute the empirical astuteness of a kernel classifier $W_K$ about test point $x$, we perform a grid search over all points in $V_x^{\kappa}$ to ensure that all points in the robustness region are labeled correctly.  

For each classifier, we measure the empirical astuteness by using three trials of $20$ test points and taking the average. While this is a relatively small amount of test data, it suffices as our purpose is to just verify that the algorithm roughly converges towards the optimal possible astuteness. Recall that for any \ncons\emph{ }consistent algorithm, as $n \to \infty$, $A_{\mathcal{V}^\kappa}$ should converge towards $A^*$, the accuracy of the Bayes optimal classifier, for \textit{any} $0 < \kappa < 1$. Thus, to verify this holds, we use $\kappa = 0.1, 0.3, 0.5$. For each of these values, we plot the empirical astuteness as the training sample size $n$ gets larger and larger. As a baseline, we also plot their standard accuracy on the test set. 

\textbf{Results and Discussion:} The results are presented in Figure~\ref{fig:validation}; the left panel is for the exponential kernel, while the right one is for the polynomial kernel. As predicted by our theory, we see that in all cases, the exponential kernel converges towards the maximum astuteness regardless of the value of $\kappa$: the only difference is that the rate of convergence is slower for larger values of $\kappa$. This is, of course, expected because larger values of $\kappa$ entail larger robustness regions. 

By contrast, the polynomial kernel performs progressively worse for larger values of $\kappa$. This kernel was selected specifically to violate the conditions of Corollary \ref{cor:kern}, and in particular fails criteria 3. However, note that the polynomial kernel nevertheless performs will with respect to accuracy thus giving another example demonstrating the added difficulty of \ncons\emph{ }consistency.

Our results bridge the gap between our asymptotic theoretical results and finite sample regimes. In particular, we see that kernel classifiers that meet the conditions of Corollary \ref{cor:kern} are able to converge in astuteness towards the \natural\emph{ }Bayes optimal classifier, while classifiers that do not meet these conditions fail. 

\section{Related Work}\label{sec:rel_work}

There is a wealth of literature on robust classification, most of which impose the same robustness radius $r$ on the entire data.  \cite{Carlini17, Liu17, Papernot17, Papernot16,Szegedy14, Hein17,Katz17,Schmidt18,Wu16,Steinhardt18, Sinha18}, among others, focus primarily on neural networks, and robustness regions that are $\ell_1, \ell_2, $ or $\ell_\infty$ norm balls of a given radius $r$. 

\cite{ChenLeiChen20} and~\cite{mma20} show how to train neural networks with different robustness radii at different points by trading off robustness and accuracy; their work differ from ours in that they focus on neural networks, their robustness regions are still norm balls, and that their work is largely empirical.

Our framework is also related to large margin classification -- in the sense that the robustness regions $\U$ induce a {\em{margin constraint}} on the decision boundary. The most popular large margin classifier is the Support Vector Machine\cite{cortes95, Bennett00, Freund99} -- a large margin linear classifier that minimizes the worst-case margin over the training data. Similar ideas have also been used to design classifiers that are more flexible than linear; for example, \cite{Luxburg03} shows how to build large margin Lipschitz classifiers by rounding globally Lipschitz functions. Finally, there has also been purely empirical work on achieving large margins for more complex classifiers -- such as~\cite{Samy18} for deep neural networks that minimizes the worst case margin, and~\cite{Weinberger05} for metric learning to find large margin nearest neighbors. Our work differs from these in that our goal is to ensure a high enough local margin at each $x$, (by considering the \natural\emph{ }regions $V_x$) as opposed to optimizing a global margin. 




Finally, our analysis builds on prior work on robust classification for non-parametric methods in the standard framework. \cite{Amsaleg17, Sitawarin19, WJC18, YRWC19} provide adversarial attacks on non-parametric methods. Wang et. al. \cite{WJC18} develops a defense for $1$-NN that removes a subset of the training set to ensure higher robustness. Yang et. al~\cite{YRWC19} proposes the $r$-optimal classifier -- which is the maximally astute classifier in the standard robustness framework -- and proposes a defense called Adversarial Pruning.


Theoretically, \cite{Bhattacharjee20} provide conditions under which weight functions converge towards the $r$-optimal classifier in the large sample limit. They show that for $r$-separated distributions, where points from different classes are at least distance $2r$ or more apart, nearest neighbors and kernel classifiers satisfy these conditions. In the more general case, they use Adversarial Pruning as a preprocessing step to ensure that the training data is $r$-separated, and show that this preprocessing step followed by nearest neighbors or kernel classifiers leads to solutions that are robust and accurate in the large sample limit. Our result fundamentally differs from theirs in that we analyze a different algorithm, and our proof techniques are quite different. In particular, the fundamental differences between the $r$-optimal classifier and the \natural\emph{ }Bayes optimal classifier call for different algorithms and different analysis techniques.

In concurrent work, \cite{ruth} proposes a similar limit to the neighborhood preserving Bayes optimal which they refer to as the margin canonical Bayes. However, their work then focuses on a data augmentation technique that leads to convergence whereas we focus on proving the neighborhood consistency of classical non-parametric classifiers.

\section*{Acknowledgments}

We thank NSF under CNS 1804829 for research support.



\bibliography{refs}
\bibliographystyle{plain}

\section*{Checklist}

The checklist follows the references.  Please
read the checklist guidelines carefully for information on how to answer these
questions.  For each question, change the default \answerTODO{} to \answerYes{},
\answerNo{}, or \answerNA{}.  You are strongly encouraged to include a {\bf
justification to your answer}, either by referencing the appropriate section of
your paper or providing a brief inline description.  For example:
\begin{itemize}
  \item Did you include the license to the code and datasets? \answerYes{See Section~\ref{gen_inst}.}
  \item Did you include the license to the code and datasets? \answerNo{The code and the data are proprietary.}
  \item Did you include the license to the code and datasets? \answerNA{}
\end{itemize}
Please do not modify the questions and only use the provided macros for your
answers.  Note that the Checklist section does not count towards the page
limit.  In your paper, please delete this instructions block and only keep the
Checklist section heading above along with the questions/answers below.

\begin{enumerate}

\item For all authors...
\begin{enumerate}
  \item Do the main claims made in the abstract and introduction accurately reflect the paper's contributions and scope?
    \answerYes{we express our claims through theorems}
  \item Did you describe the limitations of your work?
    \answerYes{}
  \item Did you discuss any potential negative societal impacts of your work?
    \answerYes{}
  \item Have you read the ethics review guidelines and ensured that your paper conforms to them?
    \answerYes{}
\end{enumerate}

\item If you are including theoretical results...
\begin{enumerate}
  \item Did you state the full set of assumptions of all theoretical results?
    \answerYes{In the theorem statements}
	\item Did you include complete proofs of all theoretical results?
    \answerYes{in the appendix}
\end{enumerate}

\item If you ran experiments...
\begin{enumerate}
  \item Did you include the code, data, and instructions needed to reproduce the main experimental results (either in the supplemental material or as a URL)?
    \answerYes{in the appendix}
  \item Did you specify all the training details (e.g., data splits, hyperparameters, how they were chosen)?
    \answerYes{Many details are given in the main body, but a full explanation with all details is in the appendix.}
	\item Did you report error bars (e.g., with respect to the random seed after running experiments multiple times)?
    \answerYes{In the appendix: this was not particularly needed for our very light experiments.}
	\item Did you include the total amount of compute and the type of resources used (e.g., type of GPUs, internal cluster, or cloud provider)?
    \answerYes{Just a simple personal computer.}
\end{enumerate}

\item If you are using existing assets (e.g., code, data, models) or curating/releasing new assets...
\begin{enumerate}
  \item If your work uses existing assets, did you cite the creators?
    \answerNA{}
  \item Did you mention the license of the assets?
    \answerNA{}
  \item Did you include any new assets either in the supplemental material or as a URL?
    \answerNA{}
  \item Did you discuss whether and how consent was obtained from people whose data you're using/curating?
    \answerNA{}
  \item Did you discuss whether the data you are using/curating contains personally identifiable information or offensive content?
    \answerNA{}
\end{enumerate}

\item If you used crowdsourcing or conducted research with human subjects...
\begin{enumerate}
  \item Did you include the full text of instructions given to participants and screenshots, if applicable?
    \answerNA{}
  \item Did you describe any potential participant risks, with links to Institutional Review Board (IRB) approvals, if applicable?
    \answerNA{}
  \item Did you include the estimated hourly wage paid to participants and the total amount spent on participant compensation?
    \answerNA{}
\end{enumerate}

\end{enumerate}

\newpage
\appendix

\section{Further Details of Definitions and Theorems}

\subsection{Non-Parametric Classifiers}

In this section, we precisely define weight functions, histogram classifiers and kernel classifiers.

\begin{defn} \label{def:weight} 
\cite{devroye96} A \textbf{weight function} $W$ is a non-parametric classifier with the following properties.
\begin{enumerate}
	\item Given input $S = \{(x_1, y_1), (x_2, y_2,), \dots, (x_n, y_n)\} \sim \D^n$, $W$ constructs functions $w_1^S, w_2^S, \dots, w_n^S: \R^d \to [0, 1]$ such that for all $x \in \R^d$, $\sum_1^n w_i^S(x) = 1$. The functions $w_i^S$ are allowed to depend on $x_1, x_2, \dots x_n$ but must be independent of $y_1, y_2, \dots, y_n$. 
	\item $W$ has output $W_S$ defined as \[ W_S(x) = \begin{cases} 
      +1 & \sum_1^n w_i^S(x)y_i > 0 \\
      -1 & \sum_1^n w_i^S(x)y_i \leq 0 \\
   \end{cases}
\]
As a result, $w_i^S(x)$ can be thought of as the weight that $(x_i, y_i)$ has in classifying $x$.
\end{enumerate}
\end{defn}

\begin{defn}
A \textbf{histogram classifier}, $H$, is a non-parametric classification algorithm over $\R^d \times \Y$ that works as follows. For a distribution $\D$ over $\R \times \Y$, $H$ takes $S = \{(x_i, y_i): 1 \leq i \leq n\} \sim \D^n$ as input. Let $k_i$ be a sequence with $\lim_{i \to \infty} k_i = \infty$ and $\lim_{i \to \infty} \frac{k_i}{i} = 0$. $H$ constructs a set of hypercubes $C = \{c_1, c_2, \dots, c_m\}$ as follows:
\begin{enumerate}
	\item Initially $C = \{c\}$, where $S \subset c$.
	\item For $c \in C$, if $c$ contains more than $k_n$ points of $S$, then partition $c$ into $2^d$ equally sized hypercubes, and insert them into $C$.
	\item Repeat step $2$ until all cubes in $C$ have at most $k_n$ points. 
\end{enumerate}
For $x \in \R$ let $c(x)$ denote the unique cell in $C$ containing $x$. If $c(x)$ doesn't exist, then $H_S(x) = -1$ by default. Otherwise, \[ H_S(x) = \begin{cases} 
      +1 & \sum_{x_i \in c(x)} y_i > 0 \\
      -1 & \sum_{x_i \in c(x)}y_i \leq 0 \\
   \end{cases}.
\]
\end{defn}

\begin{defn}
A \textbf{partitioning rule} is a weight function $W$ over $\X \times \Y$ constructed in the following manner. Given $S = \{(x_i, y_i)\} \sim \D^n$, as a function of $\{x_1, \dots, x_n\}$, we partition $\R^d$ into regions with $A(x)$ denoting the region containing $x$.  Then, for any $x \in \R^d$ we have $$w_i^S(x) = \begin{cases}1 & x_i \in A(x) \\ 0 & \text{ otherwise}\end{cases}.$$To achieve $\sum w_i^S(x) = 1$, we can simply normalize weights for any $x$ by $\sum_1^n w_i^S(X)$.
\end{defn}

\begin{defn}
A \textbf{kernel classifier} is a weight function $W$ over $\R^d \times \Y$ constructed from function $K: \R^+ \cup \{0\} \to \R^+$ and some sequence $\{h_n\} \subset \R^+$ in the following manner. Given $S = \{(x_i, y_i)\} \sim \D^n$, we have $$w_i^S(x) = \frac{K(\frac{\d(x, x_i)}{h_n})}{\sum_{j = 1}^n K(\frac{\d(x, x_j)}{h_n})}.$$ Then, as above, $W$ has output \[ W_S(x) = \begin{cases} 
      +1 & \sum_1^n w_i^S(x)y_i > 0 \\
      -1 & \sum_1^n w_i^S(x)y_i \leq 0 \\
   \end{cases}
\]
\end{defn}

\subsection{Splitting Numbers}

We refer to definitions \ref{defn:prob_radius} and \ref{defn:splitting_number}.

The main idea behind splitting numbers is that they allow us to ensure uniform convergence properties over a weight function. To prove neighborhood consistency, it is necessary for a classifier to be correct at \textit{all} points in a given region. Consequently, techniques that consider a single point will be insufficient. The splitting number provides a mechanism for studying entire regions simultaneously. For clarity, we include a quick example in which we bound the splitting number for a given weight function.

\paragraph{Example:} Let $W$ denote any kernel classifier corresponding such that $K: \R_{\geq 0} \to \R_{\geq 0}$ is a decreasing function. For any $S \sim \D^n$, observe that the condition $w_i^S(x) \geq \beta$ precisely corresponds to $\d(x, x_i) \leq \gamma$ for some value of $\gamma$. This is because $w_i^S(x) > w_j^S(x)$ if and only if $\d(x, x_i) < \d(x, x_j)$. Thus, the regions $W_{x, \alpha, \beta}$ correspond to $\{i: \d(x, x_i) \leq \gamma\}$, where $\gamma$ is a positive real number that depends on $x, \alpha, \beta$. These sets precisely correspond to subsets of $S$ that are contained within $B(x, \gamma)$. Since balls have VC dimension at most $d+2$, by  Sauer's lemma, the number of subsets of $S$ that can be obtained in this manner is $O(n^{d+2})$. Therefore, we have that $T(W,S) = O(n^{d+2})\text{ for all }S \sim \D^n.$

\subsection{Stone's Theorem}

\begin{thm}\label{thm_stone}
\cite{Stone77} Let $W$ be weight function over $\R^d \times \Y$. Suppose the following conditions hold for any distribution $\D$ over $\R^d \times \Y$.  Let $X$ be a random variable with distribution $\D_{\R^d}$, and $S = \{(x_1, y_1), (x_2, y_2), \dots, (x_n, y_n)\} \sim \D^n$. All expectations are taken over $X$ and $S$. 

1. There is a constant $c$ such that, for every nonnegative measurable function $f$ satisfying $\mathbb{E} [f(X)] < \infty$, and $\mathbb{E} [\sum_1^n w_i^S(X)f(x_i)] \leq c \mathbb{E} [f(x)].$

2. $\forall a > 0$, $\lim_{n \to \infty} \mathbb{E}[\sum_1^n w_i^S(x)I_{||x_i - X|| > a||}] = 0.$ 

3. $\lim_{n \to \infty} \mathbb{E}[\max_{1 \leq i \leq n} w_i^S(X)] = 0.$

Then $W$ is consistent. 
\end{thm}

\section{Proofs}

\paragraph{Notation:} \begin{itemize}
	\item We let $\d$ denote our distance metric over $\R^d$. For sets $X_1, X_2 \subset \R^d$, we let $\d(X_1, X_2) = \inf_{x_1 \in X_1, x_2 \in X_2} \d(x_1, x_2)$. 
	\item For any $x \in \R^d$, $B(x, a) = \{x: \d(x, x') \leq a\}$.
	\item For any measure over $\R^d$, $\mu$, we let $supp(\mu) = \{x: \mu(B(x,a)) > 0\text{ for all }a > 0\}.$ 
	\item Given some measure $\mu$ over $\R^d$ and some $x \in \R^d$, we let $r_p(x)$ denote the probability radius (Definition \ref{defn:prob_radius}) of $x$ with probability $p$. that is, $r_p(x) = \inf \{r: \mu(B(x,r)) \geq p\}.$
	\item For weight function $W$ and training sample $S$, we let $W_S$ denote the weight function learned by $W$ from $S$.
\end{itemize}

\subsection{Proofs of Theorems \ref{thm:accuracy_margin} and \ref{thm:robust_margin}}

\begin{proof}
(Theorem \ref{thm:accuracy_margin}) Let $\D = (\mu, \eta)$ be a data distribution, and let $\mu^+, \mu^-$ be as described in section \ref{sec:def_difficulties}. Observe that for any $x \in \mu^+$, the Bayes optimal classifier and the \natural\emph{ }Bayes optimal both have the same output, and furthermore the \natural\emph{ }Bayes gives this output (by definition) throughout the entirety of $V_x$, the \natural\emph{ }robustness region of $x$. It follows that the \natural\emph{ }Bayes optimal has optimal astuteness, as desired. 
\end{proof}

\begin{proof}
(Theorem \ref{thm:robust_margin}) Let $\D = (\mu, \eta)$ be a data distribution, and assume towards a contradiction that there exists classifier $f$ which has maximal astuteness with respect towards some set of robustness regions $\U = \{U_x\}$ such that $V_x \subseteq U_x$ for all $x$. The key observation is that because $f$ has maximal astuteness, we must have $f(x) = g(x)$ for almost all points $x \sim \mu$ (where $g$ is the Bayes optimal classifier). Furthermore, for those values of $x$, we must have $g$ be robust at $x$ (meaning it uniformly outputs the same output through $U_x$).

In order for $U_x$ to be strictly larger than $V_x$ for some $x$, it \textit{necessarily} must intersect with $U_{x'}$ for some $x'$ with $g(x') \neq g(x)$, and this is what causes the contradiction: $f$ cannot be astute at both $x$ and $x'$ if they are differently labeled and their robustness regions intersect. 
\end{proof}

\subsection{Proof of Theorem \ref{thm:lower_bound}}

Let $\D = (\mu, \eta)$ be the distribution with $\mu$ being the uniform distribution over $[0, 1]$ and $\eta: [0, 1] \to [0, 1]$ be $\eta(x) = x$. For example, if $(x, y) \sim \D$, then $\Pr[y = 1| x = 0.3] = 0.3$. 

We desire to show that $k_n$-nearest neighbors is not neighborhood consistent with respect to $\D$. We begin with the following key lemma.

\begin{lem}\label{cl:delta}
For any $n > 0$, let $f_n$ denote the $k_n$-nearest neighbor classifier learned from $S \sim \D^n$. There exists some constant $\Delta > 0$ such that for all sufficiently large $n$, with probability at least $\frac{1}{2}$ over $S \sim \D^n$, there exists $x \in [0,1]$ with $\frac{1}{2} - \Delta \leq x \leq \frac{1}{2} - \frac{3\Delta}{4}$ and $f_n(x) = +1$.
\end{lem}

\begin{proof}
Let $C$ be a constant such that $k_n \leq C\log n$ for all $2 \leq n < \infty$. Set $\Delta$ as \begin{equation}\label{eqn:kl}\frac{1}{2}\log_2\frac{1}{1 - 2\Delta} + \frac{1}{2}\log_2\frac{1}{1 + 2\Delta} < \frac{1}{C}.\end{equation} Let $A \subset [0,1]$ denote the interval $[\frac{1}{2} - \Delta, \frac{1}{2} - \frac{3\Delta}{4}]$. For $S \sim \D^n$, with high probability, there exist at least $\frac{\Delta n}{8}$ instances $x_i$ that are in $A$. Let us relabel these $x_i$ as $x_1, x_2, \dots, x_m$ as $$\frac{1}{2} - \Delta \leq x_1 < x_2 < \dots < x_m \leq \frac{1}{2} - \frac{3\Delta}{4}.$$

Next, suppose that for some $i$, at least half of $y_i, y_{i+1}, \dots, y_{i + k_n - 1}$ are $+1$. Then it follows that $f_n(x) = +1$ for $x = \frac{x_{i+k_n} + x_i}{2}$ because the $k_n$ nearest neighbors of $x$ are precisely $x_i, x_{i+1}, \dots x_{i + k_n - 1}$ (as a technical note we make $x$ just slightly smaller to break the tie between $x_i$ and $x_{i + k_n}$). To lower bound the probability that this occurs for some $i$, we partition $y_1, y_2, \dots y_m$ into at least $\frac{m}{2k_n}$ disjoint groups each containing $k_n$ consecutive values of $y_i$. We then bound the probability that each group will have at least $k_n/2$ $+1$s.

Consider any group of $k_n$ $y_i$s. We have that $\Pr[y_i] = +1 = \eta(x_i) = x_i \geq \frac{1}{2} - \Delta$. Since the variables $y_i$ are independent (even conditioning on $x_i$), it follows that the probability that at least half of them are $+1$ is at least  $\Pr[\text{Bin}(k_n, \frac{1}{2} - \Delta) \geq \frac{k_n}{2}].$ For simplicity, assume that $k_n$ is even. Then using a standard lower bound for the tail of a binomial distribution (see, for example, Lemma 4.7.2 of \cite{ash_1990}), we have that $$\Pr[\text{Bin}(k_n, \frac{1}{2} - \Delta) \geq \frac{k_n}{2}] \geq \frac{1}{\sqrt{2k_n}}\exp(-k_nD(\frac{1}{2}||(\frac{1}{2} - \Delta)),$$ where $D(\frac{1}{2}||(\frac{1}{2} - \Delta)) =  \frac{1}{2}\log_2\frac{1}{1 - 2\Delta} + \frac{1}{2}\log_2\frac{1}{1 + 2\Delta}$. 

To simplify notation, let $D_\Delta = D(\frac{1}{2}||(\frac{1}{2} - \Delta))$. Then because we have $\frac{m}{2k_n}$ independent groups of $y_i$s, we have that
\begin{equation*}
\begin{split}
\Pr_{S \sim \D^n}[\exists x \in [\frac{1}{2} - \Delta, \frac{1}{2} - \frac{3\Delta}{4}]\text{ s.t. }f_n(x) = +1] &\geq 1 - (1 - \frac{1}{\sqrt{2k_n}}\exp(-k_nD_\Delta))^{\frac{m}{2k_n}} \\
&\geq 1 - \exp(-\frac{m}{2k_n\sqrt{2k_n}}e^{-k_nD_\Delta}) \\
&\geq 1 - \exp(-\frac{n\Delta}{(16C\log n)^{3/2}}e^{-CD_\Delta\log n}),
\end{split}
\end{equation*}
with the inequalities holding because $m \geq \frac{n\Delta}{8}$ and $k_n \leq C \log n$. By equation \ref{eqn:kl}, $CD_\Delta < 1$. Therefore, $\lim_{n \to \infty} \frac{n}{(2C \log n)^{3/2}}e^{-CD_\Delta\log n} = \infty$, which implies that for $n$ sufficiently large, $$\Pr_{S \sim \D^n}[\exists x \in [\frac{1}{2} - \Delta, \frac{1}{2} - \frac{3\Delta}{4}]\text{ s.t. }f_n(x) = +1] \geq \frac{1}{2},$$ as desired.
\end{proof}

We now complete the proof of Theorem \ref{thm:lower_bound}.

\begin{proof}
(Theorem \ref{thm:lower_bound}) Let $\Delta$ be as described in Lemma \ref{cl:delta}, and let $\kappa = \frac{1}{2}$. For all $x < \frac{1}{2}$, we have that $[x, \frac{2x}{3} + \frac{1}{6}] \subseteq V_x^{\kappa}$. This is because we can easily verify that all points inside that interval are closer to $x$ than they are to $\frac{1}{2}$ (and consequently all points in $\mu^+ \cup \mu^{1/2}$) by factor of $2$. It follows that for all $x \in [\frac{1}{2} - \frac{7\Delta}{8}, \frac{1}{2} - \Delta]$, $$[\frac{1}{2} - \Delta, \frac{1}{2} - \frac{3\Delta}{4}] \subseteq V_x^{\kappa}.$$ However, applying Lemma \ref{cl:delta}, we know that with probability at least $\frac{1}{2}$, there exists some point $x' \in [\frac{1}{2} - \Delta, \frac{1}{2} - \frac{3\Delta}{4}]$ such that $f_n(x') = +1$. It follows that with probability at least $\frac{1}{2}$, $f_n$ lacks astuteness at \textit{all} $x \in [\frac{1}{2} - \frac{7\Delta}{8}, \frac{1}{2} - \Delta]$. Since this set of points has total probability mass $\Delta/8$, it follows that with probability at least $\frac{1}{2}$, there is a fixed gap between $A_{\V^\kappa}(f_n, \D)$ and $A(g, \D)$ (as they differ in a region of probability mass at least $\Delta/8$). This implies that $k_n$-nearest neighbors is not \ncons\emph{ }consistent. 
\end{proof}

\subsection{Proof of Theorem \ref{thm:main}}

Let $\D = (\mu, \eta)$ is a distribution over $\R^d \times \{\pm 1\}$. We will use the following notation: let $\D^+ = \{x: \eta(x) > \frac{1}{2}\}$, $\D^- = \{x: \eta(x) < \frac{1}{2}$ and $\D_{1/2} = \{x: \eta(x) = \frac{1}{2}\}$. In particular, we have that $\D^+ = \mu^+, \D^- = \mu^-$ and $\D_{1/2} = \mu^{1/2}$. This notation serve will be convenient throughout this section since it allows us to avoid overloading the symbol $\mu$. 

To show that an algorithm is \ncons\emph{ }consistent with respect to $\D$, we must show that for any $0 < \kappa < 1$, the astuteness with respect to $\V^\kappa$ converges towards the accuracy of the Bayes optimal. To this end, we fix any $0 < \kappa < 1$ and consider $\V^\kappa$. 

For our proofs, it will be useful to have the additional assumption that the robustness regions, $V_x^\kappa$ are \textit{closed}. To obtain this, we let $\U = \{U_x\}$ where $U_x = \overline{V_x^\kappa}$. Each $U_x$ is the closure of the corresponding $V_x^{\kappa}$, and in particular we have $V_x^{\kappa} \subset U_x$. Because of this, it will suffice for us to consider $A_\U$ as opposed to $A_{\V^\kappa}$ since $A_\U(f, \D) \leq A_{\V^\kappa}(f, \D)$ for all classifiers $f$.

We now begin by first proving several useful properties of $\U$ that we will use throughout this entire section. 

\begin{lem}\label{lem:u_is_npr}
The collection of sets $\U = \{U_x\}$ defined as $U_x = \overline{V_x^\kappa}$ satisfies the following properties. 
\begin{enumerate}
	\item $U_x$ is closed for all $x$. 
	\item if $x \in \D^+$, for all $x' \in U_x$, $\d(x, x') < \d(\D^+ \cup \D_{1/2}, x')$.
	\item if $x \in \D^-$, for all $x' \in U_x$, $\d(x, x') < \d(\D^- \cup \D_{1/2}, x')$. 
	\item $U_x = \{x\}$ for all $x \in \D_{1/2}$. 
	\item $U_x$ is bounded for all $x$.
\end{enumerate}
Here $\mu^+, \mu^-, \mu^{1/2}$ are as described in section \ref{sec:def_difficulties}. 
\end{lem}

\begin{proof}
Property (1) is given the by definition, and properties (2), (3) follow from the fact that $\kappa$ is strictly less than $1$. In particular, the distance function $\rho$ is continuous and consequently all limit points of a set have distances that are limits of distances within the set. Property (4) is since $V_x^\kappa = \{x\}$ for all $x \in \D_{1/2}$. 

Finally, property (5) follows from the fact that $\kappa < 1$. As $x$ gets arbitrarily far away from $x$ the ratio of its distance to $x$ with its distance to $\mu^-$ gets arbitrarily close to $1$, and consequently there is some maximum radius $R$ so that $V_x^\kappa \subset B(x, R)$. Since $B(x, R)$ is closed, it follows that $U_x \subset B(x, R)$ as well. 
\end{proof}

Next, fix $W$ as a weight function and $t_n$ is a sequence of positive integers such that the conditions of Theorem \ref{thm:main} hold, that is: 
\begin{enumerate}
	\item $W$ is consistent (with resp. to accuracy) with resp. to $\D$.
	\item For any $0 < p < 1$, $\lim_{n \to \infty} E_{S \sim \D^n} [\sup_{x \in \R^d} \sum_1^n w_i^S(x)1_{\d(x, x_i) > r_p(x)}] = 0.$
	\item $\lim_{n \to \infty} E_{S \sim D^n}[t_n \sup_{x \in \R^d} w_i^S(x)] = 0$.
	\item $\lim_{n \to \infty} E_{S \sim D^n}\frac{\log T(W,S)}{t_n} = 0$.
\end{enumerate}

Finally, we will also make the additional assumption that $\D$ has infinite support. Cases where $\D$ has finite support can be somewhat trivially handled: when the sample size goes to infinity, we will have perfect labels for every point in the support, and consequently condition 2. will ensure that any $x' \in V_x^\kappa$ is labeled according to the label of $x$.

We also use the following notation. For any classifier $f: \R^d \to \{\pm 1\}$, we let \begin{equation}\label{eqn:cool_sets}\D_f^+ = \{x: f(x' = +1\text{ for all }x' \in U_x\},\text{ and }\D_f^- = \{x: f(x' = -1\text{ for all }x' \in U_x\}.\end{equation} These sets represent the examples that $f$ robustly labels as $+1$ and $-1$ respectively. These sets are useful since they allows us to characterize the astuteness of $f$, which we do with the following lemma.

\begin{lem}\label{lem:conversion_to_measure_thing}
For any classifier $f: \R^d \to \{\pm 1\}$, we have $$A_\U(f, \D) \geq A(g, \D) - \mu(\D^+ \setminus \D_f^+) - \mu(D^- \setminus \D_f^-),$$ where $g$ denotes the Bayes optimal classifier.
\end{lem}

\begin{proof}
By property 4 of Lemma \ref{lem:u_is_npr}, $U_x = \{x\}$ for all $x \in \D_{1/2}$. Consequently, if $x \in \D_{1/2}$, there is a $\frac{1}{2}$ chance that any classifier is astute at $(x,y)$. Using this along with the definition of astuteness, we see that 
\begin{equation*}
\begin{split}
A_\U(f, \D) &= \Pr_{(x,y) \sim \D} [f(x') = y\text{ for all }x' \in U_x] \\
&= \Pr_{(x,y) \sim \D}[y = +1\text{ and }x \in (D^+ \cap D_f^+)] + \Pr_{(x,y) \sim \D}[y = -1\text{ and }x \in (D^- \cap D_f^-)] + \frac{1}{2}\Pr_{(x,y) \sim \D}[x \in \D_{1/2}]
\end{split}
\end{equation*}
However, observe by the definitions of $\D^+, \D^-$ and $\D_{1/2}$ that $$A(g, \D) = \Pr_{(x,y) \sim \D}[y = +1\text{ and }x \in D^+] + \Pr_{(x,y) \sim \D}[y = -1\text{ and }x \in D^-] + \frac{1}{2}\Pr_{(x,y) \sim \D}[x \in \D_{1/2}].$$ Substituting this, we find that 
\begin{equation*}
\begin{split}
A_\U(f, \D) &\geq A(g, \D) - \Pr_{(x,y) \sim \D}[x \in (D^+ \setminus D_f^+)] - \Pr_{(x,y) \sim \D}[x \in (D^- \setminus D_f^-)] \\
&= A(g, \D) - \mu(\D^+ \setminus \D_f^+) - \mu(D^- \setminus \D_f^-),
\end{split}
\end{equation*}
as desired. 
\end{proof}

Lemma \ref{lem:conversion_to_measure_thing} shows that to understand how $W_S$ converges in astuteness, it suffices to understand how the regions $\D_{W_S}^+$ and $\D_{W_S}^-$ converge towards $D^+$ and $D^-$ respectively. This will be our main approach for proving Theorem \ref{thm:main}. Due to the inherent symmetry between $+$ and $-$, we will focus on showing how the region $\D_{W_S}^+$ converges towards $D^+$. The case for $-$ will be analogous. To that end, we have the following key definition. 

\begin{defn}\label{defn:covered}
Let $p, \Delta > 0.$ We say $x \in \D^+$ is $(p, \Delta)$-\textbf{covered} if for all $x' \in U_x$ and for all $x'' \in B(x', r_p(x')) \cap supp(\mu)$, $\eta(x'') > \frac{1}{2} + \Delta.$ Here $r_p$ denotes the probability radius (Definition \ref{defn:prob_radius}). We also let $\D_{p, \Delta}^+$ denote the set of all $x \in \D^+$ that are $(p, \Delta)$-covered. 
\end{defn}

If $x$ is $(p, \Delta)$-covered, it means that for all $x' \in U_x$, there is a set of points with measure $p$ around $x'$ that are both close to $x'$, and likely (with at least probability $\frac{1}{2} + \Delta$) to be labeled as $+1$. Our main idea will be to show that if $x$ is $(p, \Delta)$ covered and $n$ is sufficiently large, $x$ is likely to be in $\D_{W_S}^+$. 

We begin this process by first showing that all $x$ are $(p, \Delta)$-covered for some $p, \Delta$. To do so, it will be useful to have one more piece of notation which we will also use throughout the rest of the section. We let $$\D_{1/2}^{-} = \D^- \cup \D_{1/2} = \supp(\mu) \setminus \D^+.$$ This set will be useful, since Lemma \ref{lem:u_is_npr} implies that for all $x \in \D^+$ and for all $x' \in U_x$, $\d(x, x') < \d(\D_{1/2}^{-}, x').$ We now return to showing that all $x$ are $(p, \Delta$-covered for some $p, \Delta$.

\begin{lem}\label{lem:everything_covered}
For any $x \in \D^+$, there exists $p, \Delta > 0$ such that $x$ is $(p, \Delta)$-covered.
\end{lem}

\begin{proof}
Fix any $x$. Let $f: U_x \to \R$ be the function defined as $f(x') = \d(x', \bad) - \d(x', x)$. Observe that $f$ is continuous. By assumption, $U_x$ is closed and bounded, and consequently must attain its minimum. However, by Lemma \ref{lem:u_is_npr}, we have that $f(x') > 0$ for all $x' \in U_x$. it follows that $\min_{x' \in U_x} f(x') = \gamma$ where $\gamma > 0$.

Next, let $p = \mu(B(x, \gamma/2))$. $p > 0$ since $x \in supp(\mu)$. Observe that for any $x' \in U_x$, $r_p(x') \leq \d(x, x') + \gamma/2$, where, $r_p(x')$ denotes the probability radius of $x'$. This is because $B(x', (\d(x, x') + \gamma/2))$ contains $B(x, \gamma/2)$ which has probability mass $p$. It follows that for any $x' \in U_x$, $\d(x', \bad) \geq r_p(x') + \gamma/2$. Motivated by this observation, let $A$ be the region defined as $$A = \bigcup_{x' \in U_x} B(x', r_p(x')).$$ Then by our earlier observation, we have that $\d(A, \bad) \geq \frac{\gamma}{2}$. Since distance is continuous, it follows that $\d(\overline{A}, \bad) \geq \frac{\gamma}{2}$ as well, where $\overline{A}$ denotes the closure of $A$. 

This means that for any $x'' \in \overline{A} \cap supp(\mu)$, $\eta(x'') > \frac{1}{2}$, since otherwise $\d(\overline{A}, \bad)$ would equal $0$ (as the two sets would literally intersect). Finally, $supp(\mu)$ is a closed set (see Appendix \ref{sec:distribution_details}), and thus $\overline{A} \cap supp(\mu)$ is closed as well. Since $\eta$ is continuous (by assumption from Definition \ref{defn:general_nat_robust}), it follows that $\eta$ must maintain its minimum value over $\overline{A} \cap supp(\mu)$. It follows that there exists $2\Delta > 0$ such that $\eta(x'') \geq \frac{1}{2} + 2\Delta > \frac{1}{2} + \Delta$ for all $x'' \in \overline{A} \cap supp(\mu)$. 

Finally, by the definition of $A$, for all $x' \in U_x$, $B(x', r_p(x')) \subset A$. It consequently follows from the definition that $x$ is $(p, \Delta)$-covered, as desired. 
\end{proof}

While the previous lemma show that some $p, \Delta$ cover any $x \in \D^+$, this does not necessarily mean that there are some fixed $p, \Delta$ that cover \textit{all} $x \in \D^+$. Nevertheless, we can show that this is almost true, meaning that there are some $p, \Delta$ that cover \textit{most} $x \in \D^+$. Formally, we have the following lemma.

\begin{lem}\label{lem:most_covered}
For any $\epsilon > 0$, there exists $p, \Delta$ such that $\mu(\D^+ \setminus \D_{p, \Delta}^+) < \epsilon$, where $\D_{p, \Delta}^+$ is as defined in Definition \ref{defn:covered}. 
\end{lem}

\begin{proof}
Observe that if $x$ is $(p, \Delta)$-covered, then it is also $(p', \Delta')$-covered for any $p' < p$ and $\Delta' < \Delta$. This is because $B(x', r_{p'}(x')) \subset B(x', r_p(x))$ and $\frac{1}{2} + \Delta > \frac{1}{2} + \Delta'$. Keeping this in mind, define $$\mathcal{A} = \{\D_{1/i, 1/j}^+: i, j \in \N\}.$$ For any $x \in \D^+$, by Lemma \ref{lem:everything_covered} and our earlier observation, there exists $A \in \mathcal{A}$ such that $x \in A$. It follows that $\cup_{A \in \mathcal{A}} A = \D^+$. By applying Lemma \ref{lem:measure_lemma}, we see that there exists a finite subset of $\mathcal{A}$, $\{A_1, \dots, A_m\}$ such that $$\mu(A_1 \cup \dots \cup A_m\}) > \mu(\D^+) - \epsilon.$$ Let $A_k = \D_{1/i_k, 1/j_k}^+$ for $1 \leq k \leq m$. From our previous observation once again, we see that $\cup A_i \subset \D_{1/I, 1/J}^+$ where  $I = \max(i_k)$ and $J = \max(j_k)$. It follows that setting $p = 1/I$ and $\Delta = 1/J$ suffices. 
\end{proof}

Recall that our overall goal is to show that if $x$ is $(p, \Delta)$-covered, $n$ is sufficiently large, then $x$ is very likely to be in $\D_{W_S}^+$ (defined in equation \ref{eqn:cool_sets}). To do this, we will need to find sufficient conditions on $S$ for $x$ to be in $W_S$. This requires the following definitions, that are related to \textit{splitting numbers} (Definition \ref{defn:splitting_number}). 

\begin{defn}\label{defn:Delta_particular_split_thing}
Let $x \in \R^d$ be a point, and let $S = \{(x_1, y_1), \dots, (x_n, y_n)\}$ be a training set sampled from $\D^n$. For $0 \leq \alpha$, $0 \leq \beta \leq 1$, and $0 < \Delta < \frac{1}{2}$, we define $$W_{x, \alpha, \beta}^{\Delta, S} = \{i: \d(x, x_i) \leq \alpha, w_i^S(x) \geq \beta, \eta(x_i) > \frac{1}{2} + \Delta\}.$$
\end{defn}

\begin{defn}
Let $0 < \Delta < \frac{1}{2}$, and let $S = \{(x_1, y_1), \dots, (x_n, y_n)\}$ be a training set sampled from $\D^n$. Then we let $$W^{\Delta, S} = \{W_{x, \alpha, \beta}^{\Delta, S}: x \in \R^d, 0 \leq \alpha, 0 \leq \beta \leq 1\}.$$
\end{defn}

These convoluted looking sets will be useful for determining the behavior of $W_s$ at some $x \in \D_{p, \Delta}^+$. Broadly speaking, the idea is that if every set of indices $R \subset W^{\Delta, S}$ is relatively well behaved (i.e. the number of $y_i$s that are $+1$ is close to $(|R|(\frac{1}{2} + \Delta)$, the expected amount), then $W_s(x') = +1$ for all $x' \in U_x$. Before showing this, we will need a few more lemmas.

\begin{lem}\label{lem:vc_mimicry}
Fix any $\delta > 0$ and let $0 <  \Delta < \frac{1}{2}$. There exists $N$ such that for all $n > N$ the following holds. With probability $1 - \delta$ over $S \sim \D^n$, for all $R \in W^{\Delta, S}$ with $|R| > t_n$, $\frac{1}{|R|} \sum_{i \in R} y_i \geq \Delta$ 
\end{lem}

\begin{proof}
The key idea is to observe that the set $W^{\Delta, S}$ and the value $T(W, S)$ are completely determined by $\{x_1, \dots, x_n\}$. This is because weight functions choose their weights only through dependence on $x_1, \dots, x_n$. Consequently, we can take the equivalent formulation of first drawing $x_1, \dots, x_n \sim \mu^n$, and then drawing $y_i$ independently according to $y_i = 1$ with probability $\eta(x_1)$ and $0$ with probability $1 - \eta(x_i)$. In particular, we can treat $y_1, \dots, y_n$ as independent from $W^{\Delta, S}$ and $T(W, S)$ conditioning on $x_1, \dots, x_n$. 

Fix any $x_1, \dots, x_n$. First, we see that $|W^{\Delta, S}| \leq T(W, S)$. This is because $W_{x, \alpha, \beta}^{\Delta, S}$ is a subset that is uniquely defined by $W_{x, \alpha, \beta}$ (see Definitions \ref{defn:Delta_particular_split_thing} and \ref{defn:splitting_number}). Second, for any $R \in W^{\Delta, S}$, observe that for all $i \in R$, $y_i$ is a binary variable in $[-1, 1]$ with expected value at least $(\frac{1}{2} + \Delta) - (\frac{1}{2} - \Delta) = 2\Delta$ (again by the definition). It follows that if $|R| \geq t_n$, by Hoeffding's inequality $$\Pr_{y_1 \dots y_n} [\sum_{i \in R} y_i < \Delta] \leq \exp \left( -\frac{2|R|^2\Delta^2}{4|R|} \right) \leq \exp \left( -\frac{t_n\Delta^2}{2} \right).$$ Since there at most $T(W, S)$ sets $R$, it follows that $$\Pr_{y_1 \dots y_n}[\sum_{i \in R} y_i < \Delta\text{ for some }R \in W^{\Delta, S}\text{ with }|R| > t_n] \leq T(W, S)\exp \left( -\frac{t_n\Delta^2}{2} \right).$$ However, by condition 4. of Theorem \ref{thm:main}, it is not difficult to see that this quantity has expectation that tends to $0$ as $n \to \infty$ (unless $T(W, S)$ uniformly equals $1$, but this degenerate case can easily be handled on its own). Thus, for any $\delta > 0$, it follows that there exists $N$ such that for all $n > N$, with probability at least $1 - \frac{\delta}{2}$, $T(W, S)\exp \left( -\frac{t_n\Delta^2}{2} \right) \leq \frac{\delta}{2}$. This value of $N$ consequently suffices for our lemma. 
\end{proof}

We now relate $\D_{W_S}^+$ (Equation \ref{eqn:cool_sets}) to $W^{\Delta, S}$ as well as the conditions of Theorem \ref{thm:main}.

\begin{lem}\label{lem:proving_it_works}
Let $S = \{(x_1, y_1), \dots, (x_n, y_n)\}$ and let $0 < \Delta \leq \frac{1}{2}$ and $0 < p < 1$ such that the following conditions hold. 
\begin{enumerate}
	\item For all $R \in W^{\Delta, S}$ with $|R| > t_n$, $\frac{1}{|R|} \sum_{i \in R} y_i \geq \Delta$. 
	\item $\sup_{x \in \R^d} \sum_1^n w_i^S(x)1_{\d(x, x_i) > r_p(x)} < \frac{\Delta}{5}$.
	\item $t_n\sup_{x \in \R^d} w_i^S(x) < \frac{\Delta}{5}$. 
\end{enumerate}
Then $\D_{p, \Delta}^+ \subseteq \D_{W_S}^+$. 
\end{lem}

\begin{proof}
Let $x \in \D_{p, \Delta}^+$, and let $x' \in U_x$ be arbitrary. It suffices to show that $W_S(x') = +1$ (as $x, x'$ were arbitrarily chosen). From the definition of $W_S$, this is equivalent to showing that $\sum_1^n w_i^S(x')y_i > 0.$ Thus, our strategy will be to lower bound this sum using the conditions given in the lemma statement. 

We first begin by simplifying notation. Since $S$ and $x'$ are both fixed, we use $w_i$ to denote $w_i^S(x')$. Since $n$ is fixed, we will also use $t$ to denote $t_n$. Next, suppose that $|\{x_1, \dots, x_n\} \cap B(x', r_p(x'))| = k$. Without loss of generality, we can rename indices such that $\{x_1, \dots, x_n\} \cap B(x', r_p(x')) \cap B(x', r_p(x')) = \{x_1, \dots, x_k\}$, and $w_1 \geq w_2 \geq \dots \geq w_k.$ 

Let $Y_j = \sum_{i=1}^j y_i$. Our main idea will be to express the sum in terms of these $Y_j$s as follows.  
\begin{equation*}
\begin{split}
\sum_1^n w_iy_i &= \sum_1^k w_iy_i + \sum_{k+1}^n w_iy_i \\
&= w_kY_k + (w_{k-1} - w_k)Y_{k-1} + \dots + (w_{t+1} -w_{t+2})Y_{t+1} + \sum_{i = 1}^t (w_i - w_{t+1})y_i + \sum_{k+1}^n w_iy_i \\
&= \underbrace{w_kY_k + \sum_{i = t+1}^{k-1} (w_i - w_{i+1})Y_i}_{\alpha} + \underbrace{\sum_{i = 1}^t (w_i - w_{t+1})y_i}_\beta + \underbrace{\sum_{k+1}^n w_iy_i}_\tau. 
\end{split}
\end{equation*}

We now bound $\alpha, \beta$ and $\tau$ in terms of $\Delta$ by using the conditions given in the lemma. We begin with $\beta$ and $\tau$, which are considerably easier to handle.

For $\beta$, we have that 
\begin{equation*}
\begin{split}
\beta = \sum_{i=1}^t (w_i - w_{t+1})y_i \geq \sum_{i=1}^t (w_i - w_{t+1})(-1) \geq -tw_1. 
\end{split}
\end{equation*}
By condition 2 of the lemma, we see that $tw_1 < \frac{\Delta}{5}$, which implies that $\beta \geq  -\frac{\Delta}{5}$.

For $\gamma$, we have that $\gamma = \sum_{k+1}^n w_iy_i \geq -\sum_{k+1}^n w_i$. However, for all $k+1 \leq i \leq n$, by definition of $k$, $\d(x', x_i) > r_p(x')$. It follows from condition 3 of the lemma that $\gamma \geq -\frac{\Delta}{5}$.

Finally, we handle $\alpha$. Recall that $x$ is $(p, \Delta)$-covered. It follows that for all $x'' \in supp(\mu) \cap B(x', r_p(x'))$, $\eta(x'') > \frac{1}{2} + \Delta$. Thus, by the definition of $k$, $\eta(x_i) > \frac{1}{2} + \Delta$ for $1 \leq i \leq k$. It follows that if $w_i > w_{i+1}$ or $i = k$, then 
\begin{equation*}
\begin{split}
W_{x', r_p(x'), w_i}^{\Delta, S} &= \{j: \d(x', x_j) \leq r_p(x'), w_j \geq w_i, \eta(x_j) > \frac{1}{2} + \Delta\} \\
&= \{1, \dots, i\}.
\end{split}
\end{equation*}

This implies that $\{1, \dots, i\} \in W^{\Delta, S}$, and consequently that $Y_i \geq i\Delta$, from condition 1 of the lemma. It follows that for all $t < i \leq k$, $(w_{i} - w_{i+1})Y_i \geq i(w_i - w_{i+1})\Delta$, and that $w_kY_k \geq kw_k\Delta$. Substituting these, we find that 
\begin{equation*}
\begin{split}
\alpha &= w_kY_k + \sum_{i = t+1}^{k-1} (w_i - w_{i+1})Y_i \\
&\geq kw_k\Delta + \sum_{i = t + 1}^{k-1} i(w_i - w_{i+1})\Delta \\
&= w_k\Delta + w_{k-1}\Delta + \dots + w_{t+1}\Delta + (t+1)w_{t+1}\Delta. \\
&\geq (1 - \sum_{1^t} w_i - \sum_{k+1}^n w_i)\Delta \\
&\geq (1 - \frac{2\Delta}{5})\Delta \\
&\geq (\frac{4\Delta}{5}),
\end{split}
\end{equation*}
with the last inequalities holding from the arguments given for $\beta$ and $\gamma$ along with the fact that $0 < \Delta \leq  \frac{1}{2}$. Finally, substituting these, we find that $\alpha + \beta + \gamma \geq \frac{4\Delta}{5} - \frac{2\Delta}{5} = \frac{2\Delta}{5} > 0$, as desired. 
\end{proof}

We are now ready to prove the key lemma that forms one half of the main theorem (the other half corresponding to $\D_{W_S}^-$). 

\begin{lem}
Let $\delta, \epsilon > 0$. There exists $N$ such that for all $n > N$, with probability $1 - \delta$ over $S \sim \D^n$, $\mu(\D^+ \setminus \D_{W_S}^+) < \epsilon$. 
\end{lem}

\begin{proof}
First, by Lemma \ref{lem:most_covered}, let $0 < p$ and $0 < \Delta$ be such that $\mu(\D^+ \setminus \D_{p, \Delta}^+) < \epsilon$. By combining Lemma \ref{lem:vc_mimicry}, condition 3 of Theorem \ref{thm:main}, and condition 2 of Theorem \ref{thm:main} respectively,  we see that there exists $N$ such that for all $n > N$, the following hold:
\begin{enumerate}
	\item With probability at least $1 - \frac{\delta}{3}$ over $S \sim \D^n$, for all $R \in W^{\Delta, S}$ with $|R| > t_n$, $\frac{1}{|R|} \sum_{i \in R} y_i \geq \Delta$. 
	\item With probability at least $1 - \frac{\delta}{3}$ over $S \sim \D^n$, $\sup_{x \in \R^d} \sum_1^n w_i^S(x)1_{\d(x, x_i) > r_p(x)} < \frac{\Delta}{5}$.
	\item With probability at least $1 - \frac{\delta}{3}$ over $S \sim \D^n$, $t_n\sup_{x \in \R^d} w_i^S(x) < \frac{\Delta}{5}$. 
\end{enumerate}
By a union bound, this implies that $p, \Delta, S$ satisfy the conditions of Lemma \ref{lem:proving_it_works} with probability at least $1 - \delta$.  Thus, applying the Lemma, we see that with probability $1 - \delta$, $\D_{p, \Delta}^+ \subset \D_{W_S}^+$. This immediately implies our claim. 
\end{proof}

By replicating all of the work in this section for $\D^-$ and $\D_{p, \Delta}^-$, we can similarly show the following: 

\begin{lem}
Let $\delta, \epsilon > 0$. There exists $N$ such that for all $n > N$, with probability $1 - \delta$ over $S \sim \D^n$, $\mu(\D^- \setminus \D_{W_S}^-) < \epsilon$. 
\end{lem}

Combining these two lemmas with Lemma \ref{lem:conversion_to_measure_thing} immediately implies that for all $\delta, \epsilon > 0$, there exists $N$ such that for all $n > N$, with probability $1- \delta$ over $S \sim \D^n$, $$A_\U(W_S, \D) \geq A(g, \D) - \epsilon.$$ Since $V_x^\kappa \subset U_x$ and since $\kappa$ was arbitrary, this implies Theorem \ref{thm:main}, which completes our proof. 

\subsection{Proof of Corollary \ref{cor:nn}}

Recall that $k_n$-nearest neighbors can be interpreted as a weight function, in which $w_i^S(x) = \frac{1}{k_n}$ if $x_i$ is one of the $k_n$ closest points to $x$, and $0$ otherwise. Therefore, it suffices to show that the conditions of Theorem \ref{thm:main} are met. 

We let $W$ denote the weight function associated with $k_n$-nearest neighbors.

\begin{lem}
$W$ is consistent.
\end{lem}

\begin{proof}
It is well known (for example \cite{Dasgupta14}) that $k_n$-nearest neighbors is consistent for $\lim_{n \to \infty} k_n = \infty$ and $\lim_{n \to \infty} \frac{k_n}{n} = 0$. These can easily be verified for our case.
\end{proof}

\begin{lem}
For any $0 < p < 1$, $\lim_{n \to \infty} \E_{S \sim \D^n}[ \sup_{x \in \R^d} \sum_1^n w_i^S(x)1_{\d(x, x_i) > r_p(x)}] = 0.$ 
\end{lem}

\begin{proof}
It suffices to show that for $n$ sufficiently large, all $k_n$-nearest neighbors of $x$ are located inside $B(x, r_p(x))$ for all $x \in \R^d$. We do this by using a VC-dimension type argument to show that all balls $B(x, r)$ contain a number of points from $S \sim \D^n$ that is close to their expectation. 

For $x \in \R^d$ and $r \geq 0$, let $f_{x, r}$ denote the $0-1$ function defined as $f_{x,r}(x') = 1_{x' \in B(x,r)}$. Let $F = \{f_{x, r}: x \in \R^d, r \geq 0\}$ denote the class of all such functions. It is well known that the VC dimension of $F$ is at most $d+2$. 

For $f \in F$, let $\E f$ denote $\E_{(x', y) \sim \D}f(x')$ and $\E_nf$ denote $\frac{1}{n}\sum_1^n f(x_i)$, where $\E_nf$ is defined with respect to some sample $S \sim \D^n$. By the standard generalization result of Vapnik and Chervonenkis (see \cite{dasgupta2007active} for a proof), we have that with probability $1- \delta$ over $S \sim \D^n$, \begin{equation}\label{eqn:vc}-\beta_n\sqrt{\E f} \leq \E f - \E_nf \leq \beta_n\sqrt{\E f}\end{equation} holds for all $f \in F$, where $\beta_n = \sqrt{(4/n)((d+2)\ln 2n + \ln(8/\delta)}.$ 

Suppose $n$ is sufficiently large so that $\beta_n \leq \frac{p}{2}$ and $\frac{k_n}{n} < \frac{p}{2}$, and suppose that equation \ref{eqn:vc} holds. Pick any $x \in \R^d$ and consider $f_{x, r}$ where $r > r_p(x)$. This implies $\E f_{x, r} \geq p$. Then by equation \ref{eqn:vc}, we see that $\E_nf \geq \frac{p}{2}$. This implies that all $k_n$ nearest neighbors of $x$ are in the ball $B(x,r)$, and that consequently $\sum_1^n w_i^S(x)1_{\d(x, x_i) > r} = 0$. Because this holds for all $x, r$ with $x \in \R^d$ and $r > r_p(x)$, it follows that equation $2$ implies that $$\sup_{x \in X} \sum_1^n w_i^S(x) 1_{\d(x, x_i) > r_p(x)} = 0.$$ Because equation \ref{eqn:vc} holds with probability at least $1 - \delta$, and $\delta$ can be made arbitrarily small, the desired claim follows. 
\end{proof}

Let $t_n = \sqrt{d k_n\log n}$.

\begin{lem}
 $\lim_{n \to \infty} E_{S \sim D^n}[t_n \sup_{x \in \R^d} w_i^S(x)] = 0$.
\end{lem}

\begin{proof}
Let $S \sim \D^n$. By the definition of $k_n$ nearest neighbors, $\sup_{x \in \R^d}w_i^S(x) = \frac{1}{k_n}$. Therefore, $t_n\sup_{x \in \R^d} w_i^S(x) = \sqrt{\frac{d\log n}{k_n}}$. By assumption 2. of corollary \ref{cor:nn}, $\lim_{n \to \infty} \frac{d \log n}{k_n} = 0$, which implies that $$\lim_{n \to \infty} \E_{S \sim D^n}[t_n\sup_{x \in \R^d} w_i^S(x)] = \lim_{n \to \infty} \sqrt{\frac{d\log n}{k_n}} = \lim_{n \to \infty} \frac{d \log n}{k_n} = 0,$$ as desired.
\end{proof}

\begin{lem}
$\lim_{n \to \infty} E_{S \sim D^n}\frac{\log T(W,S)}{t_n} = 0$.
\end{lem}

\begin{proof}
For $S \sim \D^n$, recall that $T(W, S)$ was defined as $$T(W, S) |\{W_{x, \alpha, \beta}: x \in \R^d, 0 \leq \alpha, 0 \leq \beta \leq 1\}|,$$ where $W_{x, \alpha, \beta}$ denotes $$W_{x, \alpha, \beta} = \{i: \d(x, x_i) \leq \alpha, w_i^S(x) \geq \beta\}.$$ Our goal will to be upper bound $\log T(W, S)$. 

To do so, we first need a tie-breaking mechanism for $k_n$-nearest neighbors. For each $x_i \in S$, we independently sample $z_i \in [0, 1]$ from the uniform distribution. We then tie break based upon the value of $z_i$, i.e. if $\d(x, x_i) = \d(x, x_j)$, we say that $x_i$ is closer to $x$ than $x_j$ if $z_i < z_j$. With probability $1$, no two values $z_i ,z_j$ will be equal, so this ensures that this method always works.

Let $A_{x, \alpha} = \{i: \d(x, x_i) \leq \alpha\}$ and let $B_{x, c} = \{i: z_i \leq c\}.$ The key observation is that for any $\alpha, \beta$, $W_{x, \alpha, \beta} = A_{x, \alpha} \cap B_{x, c}$ for some value of $c$. This can be seen by noting that the nearest neighbors of $x$ are uniquely determined by $\d(x, x_i)$ and $z_i$. Therefore, it suffices to bound $|A = A_{x, \alpha}: x \in \R^d, \alpha \geq 0\}|$ and $|B = \{B_{x, c}: x \in \R^d, c \geq 0\}|$. 

To bound $|A|$, observe that the set of closed balls in $\R^d$ has VC-dimension at most $d+2$. Thus by Sauer's lemma, there are at most $O(n^{d+2}$ subsets of $\{x_1, x_2, \dots, x_n\}$ that can be obtained from closed balls. Thus $|A| \leq O(n^{d+2}$. 

To bound $|B|$, we simply note that $B_{x, c}$ consists of all $i$ for which $z_i \leq c$. Since the $z_i$ can be sorted, there are at most $n+1$ such sets. Thus $|B| \leq n+1$.

Combining this, we see that $T(W, S) \leq |A||B| \leq O(n^{d+3})$. Finally, we see that $$\lim_{n \to \infty} \frac{\log T(W,S)}{t_n} = \lim_{n \to \infty} \frac{O(d\log n)}{\sqrt{k_n d \log n}} = \lim_{n \to \infty} \sqrt{\frac{O(d \log n)}{k_n}} = 0,$$ with the last inequality holding by condition 2. of Corollary \ref{cor:nn}. 

\end{proof}

Finally, we note that Corollary \ref{cor:nn} is an immediate consequence of the previous 4 lemmas as we can simply apply Theorem \ref{thm:main}.

\subsection{Proof of Corollary \ref{cor:kern}}

Let $W$ be a kernel classifier constructed from $K$ and $h_n$ such that the conditions of Corollary \ref{cor:kern} hold: that is, 
\begin{enumerate}
	\item $K: [0, \infty) \to [0, \infty)$ is decreasing and satisfies $\int_{\R^d}K(x)dx < \infty.$
	\item $\lim_{n \to \infty} h_n = 0$ and $\lim_{n \to \infty} nh_n^d = \infty$.
	\item For any $c > 1$, $\lim_{x \to \infty} \frac{K(cx)}{K(x)} = 0$.
	\item For any $x \geq 0$, $\lim_{n \to \infty} \frac{n}{\log n}K(\frac{x}{h_n}) = \infty$.
\end{enumerate}

It suffices to show that the conditions of Theorem \ref{thm:main} are met for $W$. Before doing this, we will describe one additional assumption we make for this case.

\paragraph{Additional Assumption:} We assume that $\D, \U$ are such that there exists some compact set $\X \subset \R^d$ such that for all $x \in supp(\mu)$, $U_x \subset \X$. This is primarily for convenience: observe that any distribution can be approximated arbitrarily closely by distributions satisfying these properties (as each $U_x$ is bounded by assumption). Importantly, because of this, we will note that it is possible for conditions 2. and 3. of Theorem \ref{thm:main} to be relaxed to taking supremums over $\X$ rather than $\R^d$. This is because in our proof, we only ever used these conditions in their restriction to $\bigcup_{x \in supp(\mu)} \bigcup{x' \in U_x} B(x', r_p(x'))$.

Using this assumption, we return to proving the corollary. 

\begin{lem}\label{cl:kern_consistency}
$W$ is consistent with respect to $\D$. 
\end{lem}

\begin{proof}
Condition 1. of Corollary \ref{cor:kern} imply that $K$ is a regular kernel. This together with Condition 2. implies that $W$ is consistent: a proof can be found in \cite{devroye96}. 
\end{proof}

To verify the second condition, it will be useful to have the following definition. 

\begin{defn}\label{defn:probability_epsilon_radius}
For any $p, \epsilon > 0$ and $x \in \X$, define $r_p^\epsilon$ as $$r_p^\epsilon(x) = \sup \{r: \mu(B(x, r)) - \mu(B(x, r_p(x)) \leq \epsilon\}.$$ 
\end{defn}

\begin{lem}\label{cl:rad}
For any $p, \epsilon > 0$, there exists a constant $c_p^\epsilon > 1$ such that $\frac{r_p^\epsilon(x)}{r_p(x)} \geq c_p^{\epsilon}$ for all $x \in \X$, where we set $\frac{r_p^\epsilon(x)}{r_p(x)} = \infty$ if $r_p(x) = 0$. 
\end{lem}

\begin{proof}
The basic idea is to use the fact that $\X$ is compact. Our strategy will be to analyze the behavior of $\frac{r_p^{\epsilon}(x)}{r_p(x)}$ over small balls $B(x_0, r)$ centered around some fixed $x_0$, and then use compactness to pick some finite set of balls $B(x_0, r)$. This must be done carefully because the function $x \to \frac{r_p^\epsilon(x)}{r_p(x)}$ is not necessarily continuous. 

Fix any $x_0 \in \X$. First, observe that $r_p^{\epsilon}(x_0) > r_p(x_0)$. This is because $B(x_0, r_p(x_0)) = \cap_{r > r_p(x_0)} B(x_0,r)$, and consequently $\lim_{r \downarrow r_p(x_0)} \mu(B(x_0, r)) = \mu(B(x_0, r_p(x_))).$ 

Next, define $$s_p^\epsilon(x) = \inf\{r: \mu(B(x, r_p(x)) - \mu(B(x,r)) \leq \epsilon\}.$$ We can similarly show that $r_p(x_0) > s_p^{\epsilon}(x_0)$. 

Finally, define $$r_0 = \frac{1}{3}\min(r_p^{\epsilon}(x_0) - r_p(x_0), r_p(x_0) - s_p^{\epsilon}(x_0)).$$ Consider any $x \in B^o(x_0, r_0)$ where $B^o$ denotes the open ball, and let $\alpha = \d(x_0, x)$. Then we have the following. 
\begin{enumerate}
	\item $r_p(x) \leq r_p(x_0) + \alpha$. This holds because $B(x, r_p(x_0) + \alpha)$ contains $B(x_0, r_p(x_0))$, which has probability mass at least $p$. 
	\item $r_p(x) \geq r_p(x_0) - \alpha$. This holds because if $r_p(x) < r_p(x_0) - \alpha$, then there would exists $r < r_p(x_0)$ such that $\mu(B(x_0, r)) \geq p$ which is a contradiction.
	\item $B(x_0, s_p^{\epsilon}(x_0)) \subset B(x, r_p(x)).$ This is just a consequence of the definition of $r_0$ and the previous observation.
\end{enumerate}
By the definitions of $r_p^\epsilon$ and $s_p^\epsilon$, we see that $\mu(B(x_0, r_p^{\epsilon}(x_0)) - \mu(B(x_0, s_p^\epsilon(x_0)) \leq 2\epsilon$. By the triangle inequality, $B(x, r_p^{\epsilon}(x_0) - \alpha) \subset B(x_0, r_p^\epsilon(x_0))$ and $B(x_0, s_p^{\epsilon}(x_0)) \subset B(x, r_p(x))$. it follows that $$\mu(B(x, r_p^{\epsilon}(x_0) - \alpha)) - \mu(B(x, r_p(x))) \leq 2\epsilon,$$ which implies that $r_p^{2\epsilon}(x) \geq r_p^{\epsilon}(x_0) - \alpha$. Therefore we have the for all $x \in B(x_0, r_0)$, $$\frac{r_p^{2\epsilon}(x) }{r_p(x)} \geq \frac{r_p^{\epsilon}(x_0) - \alpha}{r_p(x_0) + \alpha} \geq \frac{2r_p^\epsilon(x_0) + r_p(x_0)}{r_p^\epsilon(x_0) + 2r_p(x_0)}.$$ Notice that the last expression is a constant that depends only on $x_0$, and moreover, since $r_p^\epsilon(x_0) > r_p(x_0)$, this constant is strictly larger than $1$. Let us denote this as $c(x_0)$. Then we see that $\frac{r_p^{2\epsilon}(x)}{r_p(x)} \geq c(x_0)$ for all $x \in B^o(x_0, r_0)$. 

Finally, observe that $\{B^o(x_0, r_0): x_0 \in \X\}$ forms an open cover of $\X$ and therefore has a finite sub-cover $C$. Therefore, taking $c = \min_{B^o(x_0, r_0) \in C}c(x_0)$, we see that $\frac{r_p^{2\epsilon}(x)}{r_p(x)} \geq c > 1$ for all $x \in \X$. Because $\epsilon$ was arbitrary, the claim holds.
\end{proof}

\begin{lem}\label{cl:kern_radius}
For any $0 < p < 1$, $\lim_{n \to \infty} \E_{S \sim \D^n}[ \sup_{x \in \X} \sum_1^n w_i^S(x)1_{\d(x, x_i) > r_p(x)}] = 0.$ 
\end{lem}

\begin{proof}
Fix $p > 0$, and fix any $\epsilon, \delta > 0$. Pick $n$ sufficiently large so that the following hold.
\begin{enumerate}
	\item Let $c_p^\epsilon$ be as defined from Lemma \ref{cl:rad}. \begin{equation}\label{eqn:ratio}\sup_{x \in \X} \frac{K(c_p^\epsilon r_p(x)/h_n)}{K(r_p(x) / h_n)} < \delta.\end{equation} This is possible because of conditions 2. and 3. of Corollary \ref{cor:kern}, and because the function $x \to r_p(x)$ is continuous.
	\item With probability at least $1 - \delta$ over $S \sim \D^n$, for all $r > 0$, and $x \in \X$, \begin{equation}\label{eqn:uniform_vc}|\mu(B(x,r)) - \frac{1}{n}\sum_1^n 1_{x_i \in B(x, r)}| \leq \epsilon.\end{equation} This is possible because the set of balls $B(x,r)$ has VC dimension at most $d+2$.
\end{enumerate}
We now bound $\E_{S \sim \D^n}[ \sup_{x \in \X} \sum_1^n w_i^S(x)1_{\d(x, x_i) > r_p(x)}]$ by dividing into cases where $S$ satisfies and doesn't satisfy equation \ref{eqn:uniform_vc}. 

Suppose $S$ satisfies equation \ref{eqn:uniform_vc}. By condition 1. of Corollary \ref{cor:kern}, $K$ is decreasing, and by Lemma \ref{cl:rad}, $r_p^\epsilon(x) \geq c_p^\epsilon r_p(x)$. Therefore, we have that for any $x \in \X$,
\begin{equation*}
\begin{split}
\sum_1^n K(\d(x, x_i)/h_n)1_{\d(x, x_i) \geq r_p^\epsilon(x)} &\leq \sum_1^n K(c_p^\epsilon r_p(x)/h_n)\\
&\leq n\delta K(r_p(x)/h_n)),
\end{split}
\end{equation*}
where the second inequality comes from equation \ref{eqn:ratio}. 

Next, by the definition of $r_p^\epsilon(x)$, we have that $\mu(B(x, r_p^\epsilon(x)) - \mu(B(x,r_p(x))) \leq \epsilon$. Therefore, by applying equation \ref{eqn:uniform_vc} two times, we see that for any $x \in \X$ $$\sum_1^n K(\d(x, x_i)/h_n)1_{r_p(x) < \d(x, x_i) \leq r_p^\epsilon(x)} \leq 3n\epsilon K(r_p(x)/h_n).$$ Finally, we have that $$\sum_1^n w_i^S(x) \geq \sum_1^n K(r_p(x)/h_n)1_{\d(x, x_i) \leq r_p(x)} \geq n(p - \epsilon)K(r_p(x)/h_n).$$ Therefore, using all three of our inequalities, we have that for any $x \in \X$
\begin{equation*}
\begin{split}
\sum_1^n w_i^S(x)1_{\d(x, x_i) > r_p(x)} &= \sum_1^n w_i^S(x)1_{\d(x, x_i) > r_p^\epsilon(x)} + \sum_1^n w_i^S(x)1_{r_p^\epsilon \geq \d(x, x_i) > r_p(x)} \\
&= \frac{\sum_1^n K(\d(x, x_i)/h_n)1_{\d(x, x_i) > r_p^\epsilon(x)} + \sum_1^n K(\d(x, x_i)/h_n)1_{r_p^\epsilon \geq \d(x, x_i) > r_p(x)}}{\sum_1^n K(\d(x, x_i)/h_n)} \\
&\leq \frac{ n\delta K(r_p(x)/h_n)) + 3n\epsilon K(r_p(x)/h_n)}{n(p - \epsilon)K(r_p(x)/h_n).} \\
&= \frac{\delta + 3\epsilon}{p - \epsilon}.
\end{split}
\end{equation*}
If $S$ does \textit{not} satisfy equation \ref{eqn:uniform_vc}, then we simply have $\sup_{x \in \X} \sum_1^n w_i^S(x)1_{\d(x, x_i) > r_p(x)} \leq 1$. Combining all of this, we have that 
$$E_{S \sim \D^n} \sum_1^n w_i^S(x)1_{\d(x, x_i) > r_p(x)} \leq \delta(1) + (1-\delta)\frac{\delta + 3\epsilon}{p - \epsilon}.$$ Since $\delta, \epsilon$ can be made arbitrarily small, the result follows. 
\end{proof}

By assumption, $\X$ is compact and therefore has diameter $D < \infty$. Define $$t_ n = \sqrt{n\log nK(\frac{D}{h_n})}\text{ for }1 \leq n < \infty.$$

\begin{lem}\label{cl:kern_tn}
$\lim_{n \to \infty} E_{S \sim D^n}[t_n \sup_{x \in \X} w_i^S(x)] = 0$.
\end{lem}

\begin{proof}
Because $K$ is a decreasing function, we have that $K(D / h_n) \leq K(\d(x, x_i) / h_n) \leq K(0)$. As a result, we have that for any $x \in \X$, 
\begin{equation*}
\begin{split}
t_n\sup_{1 \leq i \leq n}w_i^S(x) &= \frac{t_n\sup_{1 \leq i \leq n}K(\d(x, x_i)/h_n)}{\sum_1^n K(\d(x, x_i)/h_n)} \\
&\leq \frac{t_n K(0)}{nK(D/h_n)} \\
&= K(0)\sqrt{\frac{n\log nK(D/h_n)}{n^2K(D/h_n)^2}} \\
&= K(0)\sqrt{\frac{\log n}{nK(D/h_n)}}.
\end{split}
\end{equation*}
However, by condition 4. of Corollary \ref{cor:kern}, $\lim_{n \to \infty} \frac{n}{\log n}K(D/h_n) = \infty$. Therefore, since the above inequality holds for all $ x\in \X$, we have that $$\lim_{n \to \infty} E_{S \sim D^n}[t_n \sup_{x \in \X} w_i^S(x)] \leq \lim_{n \to \infty} K(0)\sqrt{\frac{\log n}{nK(D/h_n)}} = 0.$$
\end{proof}

\begin{lem}\label{cl:kern_vc}
$\lim_{n \to \infty} E_{S \sim D^n}\frac{\log T(W,S)}{t_n} = 0$.
\end{lem}

\begin{proof}
For $S \sim \D^n$, recall that $T(W, S)$ was defined as $$T(W, S) |\{W_{x, \alpha, \beta}: x \in \X, 0 \leq \alpha, 0 \leq \beta \leq 1\}|,$$ where $W_{x, \alpha, \beta}$ denotes $$W_{x, \alpha, \beta} = \{i: \d(x, x_i) \leq \alpha, w_i^S(x) \geq \beta\}.$$ Our goal will to be upper bound $\log T(W, S)$. 

The key observation is that $W_{x, \alpha, \beta}$ is precisely the set of $x_i$ for which $\d(x, x_i) \leq r$ where $r$ is some threshold. This is because the restriction that $w_i^S(x) \geq \beta$ can be directly translated into $\d(x, x_i) \leq r$ for some value of $r$, as $K$ is a monotonically decreasing function. Thus, $T(W,S)$ is the number of subsets of $S$ that can be obtained by considering the interior of some ball $B(x,r)$ centered at $x$ with radius $r$.

We now observe that the set of closed balls in $\R^d$ has VC-dimension at most $d+2$. Thus by Sauer's lemma, there are at most $O(n^{d+2}$ subsets of $\{x_1, x_2, \dots, x_n\}$ that can be obtained from closed balls. Thus $T(W,S) \leq O(n^{d+2}$. 

Finally, we see that $$\lim_{n \to \infty} \frac{\log T(W,S)}{t_n} = \lim_{n \to \infty} \frac{O(d\log n)}{\sqrt{n\log nK(\frac{D}{h_n})}} \leq \lim_{n \to \infty} \sqrt{\frac{O(d \log n)}{nK(\frac{D}{h_n})}} = 0,$$ with the last equality holding by condition 4. of Corollary \ref{cor:kern}. 
\end{proof}

Finally, we note that Corollary \ref{cor:kern} is an immediate consequences of Lemmas \ref{cl:kern_consistency}, \ref{cl:kern_radius}, \ref{cl:kern_tn}, and \ref{cl:kern_vc}, as we can simply apply Theorem \ref{thm:main}.

\section{Useful Technical Definitions and Lemmas}\label{sec:useful_lemmas}

\begin{lem}\label{lem:measure_lemma}
Let $\mu$ be a measure over $\R^d$, and let $\mathcal{A}$ denote a countable collections of measurable sets $A_i$ such that $\mu(\bigcup_{A \in \mathcal{A}} A) < \infty$. Then for all $\epsilon > 0$, there exists a finite subset of $\mathcal{A}$, $\{A_1, \dots, A_m\}$ such that $$\mu(A_1 \cup A_2 \cup \dots \cup A_m) > \mu(\bigcup_{A \in \mathcal{A}} A) - \epsilon.$$ 
\end{lem}

\begin{proof}
Follows directly from the definition of a measure.
\end{proof}

\subsection{The support of a distribution}\label{sec:distribution_details}

Let $\mu$ be a probability measure over $\R^d$.

\begin{defn}
The \textbf{support} of $\mu$, $\supp(\mu)$, is defined as all $x \in \R^d$ such that for all $r > 0$, $\mu(B(x, r)) > 0$. 
\end{defn}

From this definition, we can show that $supp(\mu)$ is closed.

\begin{lem}
$supp(\mu)$ is closed.
\end{lem}

\begin{proof}
Let $x$ be a point such that $B(x, r) \cap \supp(\mu) \neq \emptyset$ for all $r > 0$. It suffices to show that $x \in supp(\mu)$, as this will imply closure. 

Let $x$ be such a point, and fix $r > 0$. Then there exists $x' \in B(x, r/2)$ such that $x' \in supp(\mu)$. By definition, we see that $\mu(B(x', r/3)) > 0$. However, $B(x', r/3) \subset B(x, r)$ by the triangle inequality. it follows that $\mu(B(x,r)) > 0$. Since $r$ was arbitrary, it follows that $x \in supp(\mu)$.
\end{proof}

\section{Experiment Details}\label{sec:experiment_details}

\begin{figure}
    \centering
        \includegraphics[scale=0.34] {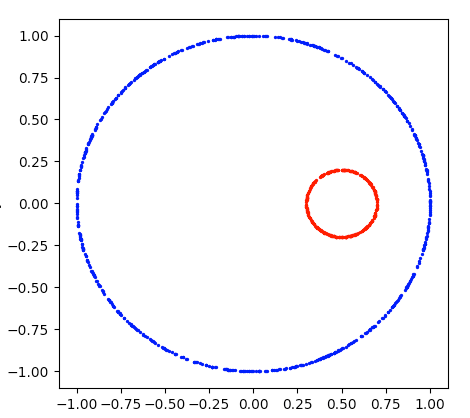}
    \caption{Our data distribution $\D = (\mu, \eta)$ with $\mu^+$ shown in blue and $\mu^-$ shown in red. Observe that this simple distribution captures varying distances between the red and blue regions, which necessitates having varying sizes for robustness regions. }
    \label{fig:distribution}
\end{figure}

\paragraph{Data Distribution} Our data distribution $\D = (\mu, \eta)$ is over $\R^2 \times \{\pm 1\}$, and is defined as follows. We let $\mu^+$ consist of a uniform distribution over the circle $x^2 + y^2 = 1$, and $\mu^-$ consist of the uniform distribution over the circle $(x-0.5)^2 + y^2 = 0.04$. The two distributions are weighted so that we draw a point from $\mu^+$ with probability 0.7, and $\mu^-$ with probability $0.3$. Finally, we utilize label noise $0.2$ meaning that the label $y$ matches that given by the Bayes optimal with probability $0.2$. In summary, $\D$ can be described with the following 4 cases:
\begin{enumerate}
	\item With probability $0.7 \times 0.8$, we select $(x,y)$ with $x \in \mu^+$ and $y = +1$.
	\item With probability $0.7 \times 0.2$, we select $(x,y)$ with $x \in \mu^+$ and $y = -1$. 
	\item With probability $0.3 \times 0.8$, we select $(x,y)$ with $x \in \mu^-$ and $y = -1$.
	\item With probability $0.3 \times 0.2$, we select $(x,y)$ with $x \in \mu^-$ and $y = +1$.
\end{enumerate}
We also include a drawing (Figure \ref{fig:distribution}) of the support of $\D$, with the positive portion $\mu^+$ shown in blue and the negative portion, $\mu^-$ shown in red. 

\paragraph{Computing Robustness Regions} 

Recall that in order to measure robustness, we utilize the so-called partial \natural\emph{ }regions $V_x^\kappa$ (Definition \ref{def:partial_nat_region}) for varying values of $\kappa$. In the case of our data distribution $\D$, $V_x^\kappa$ consists of points closer to $x$ by a factor of $\kappa$ than they are to $\mu^-$ (resp. $\mu^+$) when $x \in \mu^+$ (resp. $\mu^-$). To represent a region $V_x^\kappa$, we simply use a function $f$ that verifies whether a given point $x' \in V_x^\kappa$. While this methodology is not sufficient for training general classifiers (for a whole litany of reasons: to begin with it assumes full knowledge of the distribution), it will suffice for our toy synthetic experiments.

\paragraph{Trained Classifiers} We train two classifiers, both of which are kernel classifiers. 

The first classifier is an exponential kernel classifier with bandwidth function $h_n = \frac{1}{10\sqrt{\log n}}$ and kernel function $K(x) = e^{-x}$. 

The second classifier is a polynomial kernel classifier with bandwidth function $h_n = \frac{1}{10n^{1/3}}$ and kernel function $K(x) = \frac{1}{1 + x^2}$. 

Both of these kernels are regular kernels, and both bandwidths satisfy sufficient conditions for consistency with respect to accuracy.  In other words, both of these classifiers will converge towards the accuracy of the Bayes optimal.

However, the first classifier is selected to satisfy the criterion of Corollary \ref{cor:kern}, whereas the second is not. This distinction is reflected in our experiments.

\paragraph{Verifying Robustness} To verify the robustness of classifier $f$ at point $x$ (with respect to $V_x^\kappa$), we simply do a grid search with grid parameter 0.01. We grid the entire regions into points with distance at most $0.01$ between them, and then verify that $f$ has the desired value at all of those points. To ensure proper robustness, we also simply verify that $f$ cannot change enough within a distance of $0.01$ by constructing an upper bound on how much $f$ can possibly change. For kernel classifiers, this is simple to do as there is a relatively straightforward upper bound on the gradient of a Kernel classifier.

\end{document}